\documentclass[11pt]{article}
\usepackage[utf8]{inputenc}

\usepackage[margin=1in]{geometry}

\usepackage{authblk}
\usepackage[round]{natbib}

\usepackage[svgnames]{xcolor}
\definecolor{yale}{RGB}{14,77,146}
\usepackage[colorlinks=true,citecolor=yale]{hyperref}
\usepackage{graphicx}
\hypersetup{
  colorlinks,
  linkcolor=Black,
  urlcolor=Black}

\usepackage{url}            
\usepackage{booktabs}       
\usepackage{amsfonts}       
\usepackage{nicefrac}       
\usepackage{microtype}      

\usepackage{amsmath}
\usepackage{amsthm}
\usepackage{amssymb}
\usepackage{bbm}
\usepackage{paralist}
\usepackage{enumitem}
\usepackage{theoremref}
\usepackage{mathtools}
\RequirePackage{algorithmic}
\usepackage[ruled,linesnumbered]{algorithm2e}
\usepackage{subcaption}
\usepackage{pgf}
\usepackage{tikz}  
\usetikzlibrary{arrows}
\usetikzlibrary{arrows,automata}

\newtheorem{theorem}{Theorem}
\newtheorem{corollary}{Corollary}

\newtheorem{lemma}{Lemma}
\newtheorem{proposition}{Proposition}
\newtheorem{definition}{Definition}

\newcommand{\Alt}[1]{\operatorname{Alt}(#1)}

\DeclareMathOperator*{\argmax}{arg\,max}

\DeclareMathOperator{\KL}{\textnormal{KL}}

\newcommand{\kl}{\textnormal{kl}}

\newcommand{\scal}{\mathcal{S}}
\newcommand{\ocal}{\mathcal{O}}
\newcommand{\acal}{\mathcal{A}}

\newcommand{\bcal}{\mathcal{B}}

\newcommand{\starQ}[1]{Q^\star_{#1}}
\newcommand{\starV}[1]{V^\star_{#1}}

\newcommand{\EE}{\mathbb{E}}
\newcommand{\PP}{\mathbb{P}}
\newcommand{\norm}[1]{\left\lVert#1\right\rVert}

\makeatletter
\newcommand\footnoteref[1]{\protected@xdef\@thefnmark{\ref{#1}}\@footnotemark}

\begin{document}

\title{Adaptive Sampling for Best Policy Identification \\
in Markov Decision Processes}

\author[1]{Aymen Al Marjani\thanks{This work was done while Aymen Al Marjani was at KTH.}}
\author[2]{Alexandre Proutiere\thanks{\textit{Emails: aymen.al\_marjani@ens-lyon.fr, alepro@kth.se} }}
\affil[1]{UMPA, ENS Lyon}
\affil[2]{KTH, Royal Institute of Technology}
\maketitle

\begin{abstract}
We investigate the problem of best-policy identification in discounted Markov Decision Processes (MDPs) when the learner has access to a generative model. The objective is to devise a learning algorithm returning the best policy as early as possible. We first derive a problem-specific lower bound of the sample complexity satisfied by any learning algorithm. This lower bound corresponds to an optimal sample allocation that solves a non-convex program, and hence, is hard to exploit in the design of efficient algorithms. We then provide a simple and tight upper bound of the sample complexity lower bound, whose corresponding nearly-optimal sample allocation becomes explicit. The upper bound depends on specific functionals of the MDP such as the sub-optimality gaps and the variance of the next-state value function, and thus really captures the hardness of the MDP. Finally, we devise KLB-TS (KL Ball Track-and-Stop), an algorithm tracking this nearly-optimal allocation, and provide asymptotic guarantees for its sample complexity (both almost surely and in expectation). The advantages of KLB-TS against state-of-the-art algorithms are discussed and illustrated numerically.
\end{abstract}

\section{INTRODUCTION}\label{sec:intro}

Reinforcement Learning (RL) algorithms are designed to interact with an unknown stochastic dynamical system, and through this interaction, to identify, as fast as possible, an optimal control policy. The efficiency of these algorithms is usually measured through their {\it sample complexity}, defined as the number of samples (the number of times the algorithm interacts with the system) required to identify an optimal policy with some prescribed levels of accuracy and certainty. This paper, as most related work in this field, focuses on systems and control objectives that are modelled as a standard discounted Markov Decision Processes (MDPs) with finite state and action spaces. Various interaction models have been investigated, but sample complexity analyses have been mainly conducted under the so-called {\it generative model}, where in each step, the algorithm may sample a transition and a reward from any given (state, action) pair. We also restrict our attention to this model. 

We investigate the design of RL algorithms with minimal sample complexity. This problem has attracted a lot of attention over the last two decades. Most studies follow a minimax approach. For example, it is known \cite{azar2013minimax} that for the worst possible MDP, identifying an $\varepsilon$-optimal policy with probability $1-\delta$ requires at least ${SA\over \varepsilon^2(1-\gamma)^3}\log({SA\over \delta})$ samples, where $S$ and $A$ are the number of states and actions, respectively, and $\gamma$ is the discount factor. Note that to obtain this sample complexity lower bound, one needs to design a very specific worst-case MDP (in particular, its transition probabilities must depend on $\varepsilon$ and $\gamma$). Since the aforementioned minimax lower bound appeared, most researchers have been aiming at devising algorithms matching this bound. In contrast, we are interested in analyzing the minimal {\it problem-specific} sample complexity. Specifically, we seek to understand the dependence of the sample complexity on the MDP that has to be learnt. Problem-specific performance metrics are much more informative than their minimax counterparts, because they encode and express the inherent hardness of the MDP. Minimax metrics just represent the hardness of the worst MDP. In particular, establishing that the sample complexity of an algorithm does not exceed the minimax lower bound just reveals that the algorithm performs well for this worst MDP. However, it does not indicate whether the algorithm {\it adapts} to the hardness of the MDP, i.e., whether the optimal policy of a very easy MDP would be learnt very quickly. As a matter of fact, an algorithm with sample complexity matching the minimax lower bound just consists in sampling (state, action) pairs uniformly at random, and is not adapting to the MDP.       

The problem-specific sample complexity of identifying the best arm in stochastic Multi-Armed Bandit (MAB) problems is now well understood \cite{garivier16a}.
In this work, we explore whether the methodology used in \cite{garivier16a} for MAB problems can be extended to RL problems. This methodology consists in first deriving a problem-specific sample complexity lower bound. The latter should reveal the sample allocation leading to the minimal sample complexity. One may then devise a {\it track-and-stop} algorithm that (i) tracks the optimal sample allocation identified in the lower bound, and (ii) stops when the information gathered is judged sufficient to get the desired PAC guarantees. As it turns out, extending this methodology to RL problems raises fundamental issues, mainly due to the difficulty of computing the sample allocation leading to the minimal problem-specific sample complexity. We propose a set of tools to solve these issues. Our contributions are as follows: 

1. We derive a problem-specific sample complexity lower bound for identifying an optimal policy in a given MDP $\phi$. This bound is expressed as $T^*(\phi)\log(1/\delta)$, where the {\it characteristic time} $T^*(\phi)$ encodes the hardness of the MDP $\phi$. $T^*(\phi)$ is the value of a complex non-convex optimization problem. This complexity makes the design of a track-and-stop algorithm similar to that proposed in \cite{garivier16a} and achieving the sample complexity lower bound elusive. To circumvent this difficulty, we derive an explicit upper bound $U(\phi)$ of $T^*(\phi)$. The advantage of $U(\phi)$ is two-fold: (i) $U(\phi)$ remains problem-specific, and explicitly depends on functionals of the MDP characterizing its hardness. (ii) $U(\phi)$ corresponds to an explicit and simple sample allocation. This allows us to devise a procedure that tracks this allocation.

2. Based on our upper bound analysis, we devise KLB-TS (KL Ball Track-and-Stop), an algorithm whose sample complexity is at most $U(\phi)\log(1/\delta)$. Our algorithm relies on a procedure tracking the sample allocation leading to $U(\phi)$, and a stopping rule that we refer to as KL Ball Stopping rule because of its analogy to the way we derive the upper bound $U(\phi)$.        

3. We highlight the differences of our design approach compared to that leading to BESPOKE \cite{BESPOKE2019}, a recently proposed adaptive algorithm. As it turns out, the adaptive part of BESPOKE is very limited in practice (see related work and Appendix \ref{sec:compare} for details), and KLB-TS exhibits a much better performance numerically.


\section{RELATED WORK}\label{sec:related}

Most work on the best policy identification in MDPs with a generative model adopt a minimax approach \cite{Kearns1999}, \cite{kakade2003sample}, \cite{even2006action}, \cite{azar2013minimax}, \cite{NIPS2018_7765}, \cite{pmlr-v125-agarwal20b}, \cite{li2020breaking}. In the most recent of these papers \cite{li2020breaking}, the authors propose an algorithm whose sample complexity achieves the minimax lower bound of \cite{azar2013minimax} for a wide range of values of $\varepsilon$, namely for $\varepsilon\in (0,\frac{1}{1-\gamma}]$. Refer to the appendix for a detailed account on the minimax framework. 

As far as we are aware, the only paper attempting to propose a problem-specific analysis of the best policy identification in MDPs with a generative model is \cite{BESPOKE2019}. There, the authors proposed BESPOKE, an adaptive algorithm designed to find $\varepsilon$-optimal policies. BESPOKE starts by allocating an extremely large number of samples $n_{\min} = \frac{2\times 625^2 \times \gamma^2 \times S \times \log(1/\delta)}{(1-\gamma)^2}$ to each (state, action) pair. Then, at each iteration, BESPOKE solves a convex program whose objective is an upper-bound of the sub-optimality gap (in terms of the $\ell_\infty$-norm of the value function) of the empirical optimal policy. The solution of this program corresponds to the sampling strategy that the algorithm uses to halve the sub-optimality gap of the empirical policy in the next iteration. Interestingly, BESPOKE is the first algorithm with a problem-dependent sample complexity upper-bound. Note however that BESPOKE has not been tested numerically in \cite{BESPOKE2019}; we fill this gap in this paper. Because of its very long initialization phase, it turns out that the part where BESPOKE actually adapts its sample allocation is negligible in comparison of its total sample complexity. In Appendix \ref{sec:compare}, we provide a more detailed discussion on BESPOKE, and further compare the sample complexity upper bounds of KLB-TS and BESPOKE. Experiments in Section \ref{sec:experiments} show that KLB-TS significantly outperforms BESPOKE numerically.

\section{PRELIMINARIES AND NOTATION}\label{sec:prelim}

\subsection{Discounted MDPs}

We investigate the optimal control of dynamical systems modelled as an infinite time-horizon MDP with finite state space ${\cal S}$ and finite action spaces ${\cal A}_s$ for any $s\in {\cal S}$. Let $\acal = \cup_{s\in \scal} \acal_s$. The MDP is defined by its kernels: $\phi=(p_{\phi},q_{\phi})$, where $p_\phi$ captures the system dynamics and $q_{\phi}$ the random collected rewards. Specifically, $p_\phi(s'|s,a)$ denotes the probability of the system to be in state $s'$ after taking the action $a\in {\cal A}_s$ in state $s$. Let $p_\phi(s,a)=(p_\phi(s'|s,a))_{s'}$. $q_\phi(\cdot |s,a)$ or simply $q_\phi(s,a)$ is the density of the distribution of the reward collected in state $s$ when action $a$ is selected, w.r.t. some positive measure $\lambda$ with support included in $[0,1]$. Let $r_\phi(s,a)$ denote the expected reward collected in state $s$ when action $a$ is selected, $r_\phi(s,a)=\int_0^1 R q_\phi(R |s,a)\lambda(dR)$.

The objective is to identify a control policy $\pi: \scal \to \acal$ maximizing the long-term discounted reward $\mathbb{E}_\phi[\sum_{t=0}^\infty \gamma^t r_\phi(s^\pi(t), \pi(s^\pi(t))]$, where $s^\pi(t)$ is the state of the system at time $t$ under the policy $\pi$ and $\mathbb{E}_\phi[\cdot]$ represents the expectation taken w.r.t. to the randomness induced by $(p_\phi,q_\phi)$.  

We denote by $V_\phi^{\pi}$ the value function of the MDP $\phi$ when the control policy is $\pi$: for any $s$, $V_\phi^{\pi}(s) = \mathbb{E}_\phi[\sum_{t=0}^\infty \gamma^t r_\phi(s^\pi(t), \pi(s^\pi(t)) | s^\pi(0)=s]$. $V_\phi^\star$ corresponds to the value function when the policy $\pi$ is optimal. Note that since the rewards are lower and upper bounded by 0 and 1, respectively, we have for any $s$, $V_\phi^\star(s) \in [0,{1\over 1-\gamma}]$. Similarly, the $Q$-function is denoted by $Q_\phi^{\pi}$, and $\starQ{\phi}$ when $\pi$ is optimal. The sub-optimality gap of action $a$ in state $s$ is defined as $\Delta_{sa} = \starV{\phi}(s) - \starQ{\phi}(s,a)$. Finally, denote by $\Pi_\phi^\star$ the set of optimal policies for $\phi$. 

\medskip
\noindent
{\bf Assumption 1.} To simplify notation and the analysis, we assume that $\phi$ admits a unique optimal control policy denoted by $\pi_\phi^\star$. This means that $\phi \in \Phi = \{\phi : |\Pi_\phi^\star |=1 \}$.

\subsection{Best-policy identification}

We aim at devising an algorithm identifying $\pi^\star_\phi$ as quickly as possible in the fixed-confidence setting: when the algorithm stops and returns an estimated optimal policy $\hat{\pi}$, we should have $\mathbb{P}_\phi[ \hat\pi \neq \pi^\star_\phi]\le \delta$, for some pre-defined confidence parameter $\delta >0$. Such an algorithm consists of a sampling rule, a stopping rule, and a decision rule. An algorithm $\chi$ gathers information sequentially, and we denote by ${\cal F}^\chi_t$ the $\sigma$-algebra generated by all observations made under $\chi$ up to and including round $t$. 

\medskip
\noindent
{\bf Sampling rule.} In round $t$, the algorithm $\chi$ selects a (state, action) pair $(s_t,a_t)$ to explore, depending on past observations. $(s_t,a_t)$ is ${\cal F}_{t-1}^\chi$-measurable. $\chi$ observes the next state denoted by $s_t'$ and a random reward $R_t$. Note that any admissible (state, action) pair may be selected (we consider a generative model).

\medskip
\noindent
{\bf Stopping and decision rules.} After gathering enough information, $\chi$ may decide to stop sampling and to return an estimated best policy. The algorithm stops after collecting $\tau$ samples, and $\tau$ is a stopping time w.r.t. the filtration $({\cal F}_t^\chi)_{t\ge 1}$. The estimated best policy $\hat\pi$ is then ${\cal F}_\tau^\chi$-measurable. $\tau$ is referred to as the sample complexity of $\chi$. 

\medskip
\noindent
{\bf $\delta$-PC algorithms.} An algorithm is $\delta$-Probably Correct ($\delta$-PC) if it satisfies the two following conditions: for any MDP $\phi\in \Phi$, (i) it stops in finite time almost surely, $\mathbb{P}_\phi[\tau<\infty]=1$, and (ii) $\mathbb{P}_\phi[ \hat\pi \neq \pi^\star_\phi]\le \delta$.

\subsection{Additional notation}
$\mathbbm{1}(s)$ denotes the canonical base vector in $\mathbb{R}^{\scal}$ whose only non-zero entry is at index $s$. $\Sigma = \{\omega \in [0,1]^{S \times A}: \underset{s,a}{\sum} w_{s a} = 1 \}$ denotes the simplex in $\mathbb{R}^{S\times A}$.
The Kullback-Leibler divergence between two probability distributions $P$ and $Q$ on some discrete space $\scal$ is defined as: $KL(P\| Q) = \sum_{s \in \scal} P(s)\log(\frac{P(s)}{Q(s)})$. For Bernoulli distributions of respective means $p$ and $q$, the KL divergence is denoted by $\kl (p,q)$. For distributions over $\mathbb{R}$ defined through their densities $p$ and $q$ w.r.t. some positive measure $\lambda$, the KL divergence is:
$KL(p\| q)=\int_{-\infty }^{\infty }p(x)\log \left(\frac{p(x)}{q(x)}\right)\,\lambda(dx)$. For two MDPs $\phi$ and $\psi$, we define $\textrm{KL}_{\phi|\psi}(s,a)$ as the KL divergence between the distributions of the random observations made for the (state, action) pair $(s,a)$ under $\phi$ and $\psi$: 
\begin{align*}
\textrm{KL}_{\phi|\psi}(s,a) = \ & KL(p_\phi(s,a)\| p_\psi(s,a)) + KL(q_\phi(s,a)\| q_\psi(s,a)).
\end{align*}
%

\section{PROBLEM-SPECIFIC SAMPLE COMPLEXITY LOWER BOUND}\label{sec:lowerbound}

To derive a problem-specific sample complexity lower bound, we use classical change-of-measure arguments as those leveraged towards regret and sample complexity lower bounds \cite{lai1985, garivier16a} in bandit problems. These arguments lead to constraints on the expected numbers of times each (state, action) pair should be explored under any $\delta$-PAC algorithm. 
\begin{definition}
The set of alternative MDPs is defined as: $\Alt\phi = \{\psi\ \mathrm{ MDP}:  \Pi^\star_{\phi} \cap   \Pi^\star_{\psi} = \emptyset \}$.
\end{definition}
Let $\psi\in \Alt\phi$ be an alternative MDP and consider a $\delta$-PAC algorithm. We denote by $O_\tau$ the set of observations made under the algorithm until it stops. Further consider $L_\tau$ the log-likelihood ratio of $O_\tau$ under the MDPs $\phi$ and $\psi$. Using similar techniques as those used in the proof of Wald's first lemma, we get (all proofs are detailed in the appendix):
\begin{lemma}\label{lemma:wald} Let $n_t(s,a)$ be the number of times $(s,a)$ has been explored up to and including step $t$. For any $\phi\in \Phi$,
$
\mathbb{E}_\phi[L_\tau]=\sum_{s,a}\mathbb{E}_\phi[n_\tau(s,a)]\KL_{\phi \mid \psi} (s,a).
$ \end{lemma}

From the above lemma, and using the same arguments as in \cite{kaufmann2016complexity}, one may derive the following data processing inequality, valid for any ${\cal F}_\tau$-measurable event $E$:
$$
\sum_{s,a}\mathbb{E}_\phi[n_\tau(s,a)]\KL_{\phi \mid \psi} (s,a) \ge \kl (\mathbb{P}_\phi[E],\mathbb{P}_\psi[E]).
$$
Next, we select the event $E$ as $\{ \hat{\pi} \notin \Pi^\star(\phi)\}$. Since the algorithm is $\delta$-PAC, and since $\psi\in \Alt\phi$, we have: $\mathbb{P}_\phi[E]\le \delta$ and $\mathbb{P}_\psi[E]\ge \mathbb{P}_\psi[\hat{\pi} \in \Pi^\star(\psi)]\ge 1-\delta$.
Using the monotonicity of the KL divergence, we deduce that $\kl (\mathbb{P}_\phi[E],\mathbb{P}_\psi[E])\ge \kl(\delta,1-\delta)$. We have established that under any $\delta$-PAC algorithm, the numbers of times $(n_\tau(s,a))_{s,a}$ the different (state, action) pairs are explored satisfy:   
 for any MDP $\psi\in\Alt\phi$,
\begin{equation} \label{eq:information-bdd}
\sum_{s,a}\mathbb{E}_\phi[n_\tau(s,a)]\KL_{\phi \mid \psi} (s,a) \ge \kl (\delta,1-\delta).
\end{equation}
Combining the above constraints with the fact that $\tau = \sum_{s,a}n_\tau(s,a)$, we obtain the following sample complexity lower bound.

\begin{proposition}\label{prop:LBgeneric} The sample complexity of any $\delta$-PAC algorithm satisfies: for any $\phi\in \Phi$,
\begin{equation}\label{eq:lb1}
\mathbb{E}_\phi[\tau] \ge T^*(\phi) \kl (\delta,1-\delta),
\end{equation}
\begin{equation}
\hbox{where }   T^*(\phi)^{-1} = \underset{\omega \in \Sigma}{\sup} \ \underset{\psi \in \Alt\phi}{\inf} \sum_{s,a}\omega_{s a}\textrm{KL}_{\phi|\psi}(s,a).
\label{Prob:main}
\end{equation}
\end{proposition}

In the above proposition, $\omega_{s a}\kl (\delta,1-\delta)$ can be interpreted as the expected proportion of times the pair $(s,a)$ is explored under the algorithm. Taking the supremum over $\omega$ then corresponds to selecting an optimal sampling rule. In the following, $\omega$ is referred to as the allocation vector.
%
%
%
%
%

\subsection{Properties of the problem (\ref{Prob:main})}
We now provide useful properties of the optimization problem (\ref{Prob:main}). Additional properties of the problem are presented in Appendix \ref{sec:appProp}.

\medskip
\noindent
{\bf (i) The set of alternative MDPs.} To simplify the notation we use $\pi^\star$ instead of $\pi_{\phi}^\star$. Our first result concerns the set $\Alt\phi$ of {\it alternative} MDPs:

\begin{lemma} $\Alt\phi = \underset{s, a\neq \pi^\star(s)}{\bigcup} \mathrm{Alt}_{sa}(\phi)$ where
$$\mathrm{Alt}_{sa}(\phi)=\{\psi :  Q_{\psi}^{\pi^\star}(s,a) > V_{\psi}^{\pi^\star}(s) \}.$$ 
\label{lemma:2}
\end{lemma}

The above lemma states that an alternative MDP $\psi$ is such that $\pi^\star$, the optimal policy of $\phi$, can be improved under $\psi$ locally at some state $s$, by selecting in $s$ some previously sub-optimal action $a$, instead of $\pi^\star(s)$. Using this lemma, we can simplify the expression of the characteristic time appearing in Proposition \ref{prop:LBgeneric}. Indeed, (\ref{Prob:main}) is equivalent to: 
\begin{equation}
    \underset{\omega \in \Sigma}{\sup}  \min_{s, a\neq \pi^\star(s)}\ \underset{\psi \in \mathrm{Alt}_{sa}(\phi)}{\inf} \sum_{s',a'}\omega_{s',a'}\textrm{KL}_{\phi|\psi}(s',a'). 
\label{eq:rewritten_problem}
\end{equation}

Next, we rewrite the problem in an analytic manner. To this aim, we parametrize $\psi$ by its transition probabilities and rewards $u = (q_{\psi}(s,a), p_{\psi}(s,a))_{s,a \in \scal\times \acal}$ and introduce the following notations: for all $(s,a)$, $dr(s,a) = (r_{\psi}- r_{\phi})(s,a)$ and $dp(s,a) = (p_{\psi}- p_{\phi})(s,a)$. Further define $dV^{\pi^\star}= \left([V_{\psi}^{\pi^\star}- V_{\phi}^{\pi^\star}](s) \right)_{s \in \scal}$. 

Combining the condition : $Q_{\psi}^{\pi^\star}(s,a) > V_{\psi}^{\pi^\star}(s) $ with the fact that $Q_{\phi}^{\pi^\star}(s,a) +\Delta_{sa} = V_{\phi}^{\pi^\star}(s) $ we obtain that $\psi \in \mathrm{Alt}_{sa}(\phi)$ if and only if:
\begin{align}
\nonumber \Delta_{sa}  <\  &  dr(s,a)  + \gamma dp(s,a)^{\top} V_{\phi}^{\pi^\star}\\
&+ [\gamma p_{\psi}(s,a) - \mathbbm{1}(s)] ^{\top}dV^{\pi^\star}.
\label{eq:rewritten_conditions_with_kernels}
\end{align}
The above inequality states that for $\psi$ to be in $\mathrm{Alt}_{s a}(\phi)$, the changes in the rewards and transitions between $\phi$ and $\psi$ should be greater than the sub-optimality gap of action $a$ in state $s$. Defining ${\cal U}_{s a}=\{u :$ (\ref{eq:rewritten_conditions_with_kernels}) holds$\}$, we conclude that both the optimization problems (\ref{Prob:main}) and (\ref{eq:rewritten_problem}) are equivalent to: 
\begin{equation}
    \underset{\omega \in \Sigma}{\sup} \min_{s, a\neq \pi^\star(s)}\ \underset{u \in {\cal U}_{s a}}{\inf} \sum_{s',a'}\omega_{s',a'}\textrm{KL}_{\phi|\psi}(s',a'). 
\label{eq:analytic_problem}
\end{equation}

\medskip
\noindent
{\bf (ii) Non-convexity of the problem (\ref{Prob:main}).} The characteristic time $T^*(\phi)$, as well as the optimal sampling rule are characterized by the solution of (\ref{Prob:main}) or that of (\ref{eq:rewritten_problem}). If we think of a track-and-stop algorithm to identify the best policy (as proposed in \cite{garivier16a} for the simple MAB problem), one would need to repeatedly solve these optimization problems. It is then important to be able to do it in a computationally efficient way. Unfortunately, these problems are probably very hard to solve. This is well illustrated by the fact that the following sub-problem is not convex: 
\begin{equation}\label{eq:subprob}
T(\phi,\omega)^{-1} = \underset{\psi \in \Alt\phi}{\inf} \sum_{s,a}\omega_{s a}\textrm{KL}_{\phi|\psi}(s,a).
\end{equation}
Actually, in the example presented in Fig. \ref{fig:example1}, we can specify $\phi$ such that the sets $\Alt\phi$ and $\mathrm{Alt}_{sa}(\phi)$ are not convex. 

\begin{figure}[htb]
\center
\includegraphics[width= 0.8\textwidth]{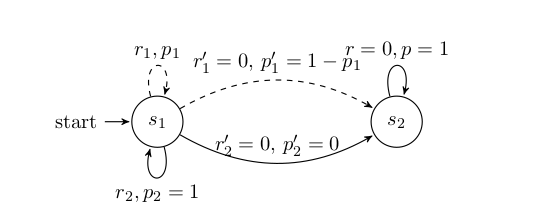}
\caption{A class of two-state MDPs, with $\gamma=0.9$. Actions $a_1$ and $a_2$ are available in state $s_1$. State $s_2$ is absorbing. Dashed (resp. full) arrows indicate the transitions when action $a_1$ (resp. $a_2$) is chosen. Numbers above each arrow indicate the transition probability and the average reward, e.g. $p_2' =\mathbb{P}[s_2 | s_1,a_2]$.}
\label{fig:example1}
\end{figure}

Consider $\phi, \psi, \overline{\psi}$ belonging to the class of MDPs specified in Fig. \ref{fig:example1}, each defined by the vector $(r_2,r_1, p_1)$ (all other parameters values are fixed as in the figure): 
\begin{equation*}
\begin{cases}
\psi = (r_2 = 0.25,\ r_1 = 0.93,\ p_1 = 0.7)\\
\overline{\psi} = (r_2 = 0.1,\ r_1 = 0.47,\ p_1 = 0.6)\\ 
\phi = \frac{\psi +\overline{\psi}}{2} = (r_2 = 0.175,\ r_1 = 0.6925,\ p_1 = 0.65)
\end{cases}
\end{equation*}
Then a simple calculation shows that the pair $(s_1,a_1)$ is optimal : $\frac{r_1}{1- \gamma p_1} > \frac{r_2}{1- \gamma p_2}$ for both $\psi$ and $\overline{\psi}$, while it is sub-optimal : $\frac{r_1}{1- \gamma p_1} < \frac{r_2}{1- \gamma p_2}$ for $\phi$. In other words, both $\psi$ and $\overline{\psi}$ are in $\Alt\phi$ and $\mathrm{Alt}_{s_1 a_1}(\phi)$ but their average is not: $\frac{\psi +\overline{\psi}}{2} = \phi \notin \Alt\phi$.
Therefore the sets $\Alt\phi$ and $\mathrm{Alt}_{s_1 a_1}(\phi)$ are not convex. Observe that this non-convexity does not arise in simple MAB problems. Indeed, there, the set of parameters (e.g., the average reward vectors $\mu=(\mu_1,\ldots,\mu_K)$) such that a given arm is optimal is always convex, i.e., $\{ \mu: \mu_k > \max_{j\ne k}\mu_j\}$ is convex.

%
%

\subsection{Upper bound of $T^*(\phi)$}

We use the analytic version (\ref{eq:analytic_problem}) of the optimization problem that defines the sample complexity lower bound to derive a simple (but still problem-specific) upper bound of the characteristic time $T^*(\phi)$. The upper bound actually corresponds to a sampling rule that is explicit, i.e., we do not need to solve any optimization problem to get it. Using this upper bound and the corresponding sampling rule, we will be able to devise a simple track-and-stop algorithm with provable performance guarantees. In addition, the upper bound has the right dependence in the sub-optimality gaps, and we also prove that it remains smaller than existing minimax sample complexity lower bounds.

Before we state the main result leading to our upper bound, we introduce additional notations. \\
$\bullet$ $\Delta_{\min} = \underset{s,a\neq \pi^\star (s)}{\min} \Delta_{sa}$ denotes the minimum sub-optimality gap in $\phi$. \\
$\bullet$ $\mathrm{Var}_{p_{\phi}(s,a)}[\starV{\phi}] = \mathrm{Var}_{s'\sim p_{\phi}(.|s,a)}[\starV{\phi}(s')]$ (resp. $\mathrm{MD}_{ p_{\phi}(s,a)}[\starV{\phi}] = \norm{\starV{\phi} - \EE_{s'\sim p_{\phi}(.|s,a)}[\starV{\phi}(s')]}_{\infty}$) is the variance (resp. maximum deviation from the mean) of the next-state value after taking  state-action pair $(s,a)$.\\
$\bullet$ $\mathrm{Var}_{\max}^{\star}[\starV{\phi}] = \underset{s}{\max}\ \mathrm{Var}_{ p_{\phi}(s,\pi^\star(s))}[\starV{\phi}]$ (resp. $\mathrm{MD}_{\max}^{\star}[\starV{\phi}] = \underset{s}{\max}\ \mathrm{MD}_{p_{\phi}(s,\pi^\star(s))}[\starV{\phi}]$) is the maximum variance (resp. maximum deviation) of the next-state value after taking an optimal action.

\begin{theorem} We have:
\begin{equation}
T^*(\phi) \leq \underset{\omega \in \Sigma}{\inf}\ \bigg(\underset{s,a\neq \pi^\star (s)}{\max}\ \frac{T_1(s,a;\phi) + T_2(s,a;\phi)}{\omega_{s a}} + \frac{T_3(\phi) + T_4(\phi)}{\underset{s}{\min}\ \omega_{s,\pi^\star(s)}} \bigg), 
\label{eq:upper_bound_problem}
\end{equation}
where
\begin{align}
&T_1(s,a;\phi) \triangleq \displaystyle{\frac{2}{\Delta_{sa}^2}}, \\
&T_2(s,a;\phi) \triangleq \max\bigg(\frac{16\mathrm{Var}_{ p_{\phi}(s,a)}[\starV{\phi}] }{\Delta_{sa}^2},\ \frac{6\mathrm{MD}_{p_{\phi}(s,a)}[\starV{\phi}]^{4/3}}{\Delta_{sa} ^{4/3}}\bigg),\\
&T_3(\phi) \triangleq \displaystyle{\frac{2}{[\Delta_{\min} (1-\gamma) ]^2}},\\
\\
&T_4(\phi) \triangleq  \min\Bigg(\frac{27}{\Delta_{\min}^2(1-\gamma)^3},\ \max\bigg(\frac{16\mathrm{Var}_{\max}^{\star}[\starV{\phi}]}{\Delta_{\min}^2 (1-\gamma)^2}, \frac{6\mathrm{MD}_{\max}^{\star}[\starV{\phi}]^{4/3}}{\Delta_{\min}^{4/3} (1-\gamma)^{4/3}} \bigg)\Bigg).
\label{eq:four_terms}
\end{align}

%

\label{theorem:pre-upper-bound}
\end{theorem}
The proof of the theorem relies on writing each of the difference terms $dr(s,a)$, $dp(s,a)$, $dr^{\pi^\star}$ and $dp^{\pi^\star}$ involved in the constraint (\ref{eq:rewritten_conditions_with_kernels}) as a proportion of the sub-optimality gap $\Delta_{sa}$. Then, using classical f-divergences inequalities, as well as a variance inequality from \cite{azar2013minimax}, we relate each difference term to the KL divergences appearing in the objective function of the problem (\ref{eq:analytic_problem}). With this perspective in mind, the terms $T_1(s,a;\phi)$ and $T_2(s,a;\phi)$ can be interpreted as the sample complexity costs to learn the reward of (state,action) pair $(s,a)$ and the corresponding transition probabilities, respectively. Similarly, the terms $T_3(\phi)$ and $T_4(\phi)$ are interpreted as the sample complexity costs to estimate the future rewards collected from the next state and the transitions from the next state. 

\begin{corollary} Let $H_{s a} \triangleq T_1(s,a;\phi) + T_2(s,a;\phi)$ and $H^\star \triangleq  S(T_3(\phi) + T_4(\phi))$. Then the solution of the problem (\ref{eq:upper_bound_problem}) is given by the unique allocation vector $\overline{\omega} \in \Sigma$ defined by ($\sim$ means {\it proportional to}): for all $s\in {\cal S}$, 
\begin{equation}
\left\{
\begin{array}{l}
\overline{\omega}_{s,\pi^\star(s)} \sim {1\over S}{\sqrt{H^\star (\sum_{s, a\neq \pi^\star(s)}H_{s a})}},\\ 
\overline{\omega}_{s,a} \sim H_{s a}, \ \ \ \hbox{for }s, a\neq \pi^\star(s).
\end{array}
\right.
\label{eq:optimal_weights}
\end{equation}
%
This allocation yields the following upper bound: 
\begin{equation}
T^*(\phi) \leq U(\phi) \triangleq 2(H^\star + \sum_{s, a\neq \pi^\star(s)}H_{s a}).
\label{eq:upper_bound}
\end{equation}
\label{corollary:upper_bound}
\end{corollary}
In the previous corollary, $\overline{\omega}_{s,a}$ is the optimal proportion of times $(s,a)$ should be sampled, and hence for $s, a\neq \pi^\star(s)$, $H_{s a}$ corresponds to the {\it hardness} of learning that $(s,a)$ is sub-optimal. It scales as the inverse of the square of the gap $\Delta_{sa}$ and is proportional to the variance of future rewards after taking $(s,a)$. 

Further observe that since the rewards are normalized, we always have: for all $(s,a)$, $\mathrm{Var}_{p_{\phi}(s,a)}[\starV{\phi}] \leq \frac{1}{(1-\gamma)^2}$ and $\mathrm{MD}_{p_{\phi}(s,a)}[\starV{\phi}]\leq \frac{1}{(1-\gamma)}$. In addition, we show in Lemma \ref{lemma:delta_min} (see Appendix E) that $\Delta_{\min}$ is always smaller than 1. These observations allow us to upper bound $T_1(s,a;\phi)$, $T_2(s,a;\phi)$, $T_3(\phi)$ and $T_4(\phi)$, and to prove the following corollary.

\begin{corollary}
We have: $U(\phi) = \ocal\left(\frac{SA}{\Delta_{\min}^2 (1-\gamma)^3}\right).$
\end{corollary}

The above result is obtained by plugging the uniform allocation $\omega_{sa}=1/SA$ in (\ref{eq:upper_bound_problem}). Hence this naive uniform allocation yields an upper bound scaling as the known minimax sample complexity lower bound $\frac{SA}{\Delta_{\min}^2 (1-\gamma)^3}$. This result also implies that a track-and-stop algorithm sampling the pairs $(s,a)$ according to $\overline{\omega}$ will perform better than the minimax bound. This algorithm will become strictly better when $\mathrm{Var}_{\max}^{\star}[\starV{\phi}] = o(1/(1-\gamma))$, i.e., when the variance of the next-state value after taking the optimal action is small.

\section{ALGORITHM}\label{sec:algo} 

In this section, we present KLB-TS (KL-Ball Track-and-Stop), an algorithm that selects the successive (state, action) pairs so as to track the allocation $\overline{\omega}$, the problem-specific allocation (\ref{eq:optimal_weights}) that leads to the upper bound (\ref{eq:upper_bound}). The algorithm is a track-and-stop, whose stopping rule does not follow a generic Generalized Likelihood Ratio Test as that used \cite{garivier16a} for MAB problems (refer to Subsection \ref{subsection:test} for detail).

The algorithm takes as input the confidence parameter $\delta$ and any black-box planner MDP-SOLVER. The latter takes as input an MDP $\phi$, and returns an optimal policy $\pi^\star_{\phi} \in \Pi^\star_{\phi}$. For practical implementations, we use the  Policy Iteration algorithm. 

KLB-TS starts exploring each (state, action) pair once, to construct an initial estimate $\widehat{\phi}$ of the true MDP $\phi$. The algorithm maintains, after $t$ collected observations, an estimate $\widehat{\phi}_t$ of the true MDP. Based on this estimate, KLB-TS computes an estimate of the allocation $\overline{\omega}$, and selects the next (state, action) pair to track it. After each observation, the estimated MDP $\widehat{\phi}_t$ is updated. Finally, the algorithm checks if a stopping condition is satisfied, in which case the algorithm stops and returns the empirical optimal policy $\widehat{\pi}^\star_{\tau}$. The stopping condition is referred to as the {\it KL-Ball stopping rule} since it is inspired by the derivation of the upper bound of $T^*(\phi)$. There, the various terms involved in the exploration constraints are upper bounded by KL divergences, i.e., are in a KL ball.

The pseudo-code of KLB-TS is presented in Algorithm \ref{main_algo1}. Its sampling and stopping rule are described in detail in the next two sub-sections.

\begin{algorithm}[h]
\caption{KLB-TS}
\begin{algorithmic}

\INPUT{Black-box planner MDP-SOLVER(), Confidence parameter $\delta$. }
 \STATE Collect one sample from each (s,a) in  $\scal\times\acal$.
 \STATE Set $t \leftarrow S A$ and $n_t(s,a) \leftarrow 1$, for all (s,a).
 \STATE Initialize empirical estimate $\widehat{\phi}_t$ of $\phi$.
 \STATE $\widehat{\pi}_t^{\star}\leftarrow \textrm{MDP-SOLVER}(\widehat{\phi}_t)$.
\WHILE{Stopping condition (\ref{eq:simplifed_stopping_rule}) is not satisfied}
\STATE Compute allocation vector $\overline{\omega}(\widehat{\phi}_t)$ of equation (\ref{eq:optimal_weights}).
\STATE Sample from $(s_{t+1},a_{t+1})$ determined by equation (\ref{eq:sampling_rule}).
\STATE For all (s,a) set:
\begin{equation*}
n_{t+1}(s,a) \leftarrow 
\begin{cases}
n_{t}(s,a) + 1 \textrm{  if $(s,a) =(s_{t+1}, a_{t+1})$} \\
n_{t}(s,a) \textrm{ Otherwise}
\end{cases}
\end{equation*}
\STATE $t\leftarrow t+1$.
\STATE Update empirical estimate $\widehat{\phi}_{t}$ of $\phi$.
\STATE $\widehat{\pi}_t^{\star}\leftarrow \textrm{MDP-SOLVER}(\widehat{\phi}_t)$.
\ENDWHILE
\OUTPUT{Empirical optimal policy $\widehat{\pi}_{\tau}^{\star}$}   
\label{main_algo1}
\end{algorithmic}
\end{algorithm}

\subsection{Sampling rule}

To build an algorithm with sample complexity matching the upper-bound of Corollary \ref{corollary:upper_bound}, the sampling proportions of (state,action) pairs should be as close as possible to the near-optimal weights defined in (\ref{eq:optimal_weights}). To this aim, we simply use the C-tracking rule defined in \cite{garivier16a}, which we recall below.
 
Define $\overline{\omega}^{\varepsilon}(\phi)$ as the $L^{\infty}$ projection of $\overline{\omega}(\phi)$ onto\\ $\Sigma^{\varepsilon} = \{\omega \in [\varepsilon,1]^{SA} :  \underset{s,a}{\sum}\ \omega_{s,a} = 1\}.$ Further define $\varepsilon_t = (S^2 A^2 + t)^{-1/2}/2$. Then the (state, action) pair to be sampled in round $t+1$ is defined as:
\begin{equation}
 (s_{t+1}, a_{t+1}) \in   \underset{(s,a)\in \scal\times\acal}{\argmax}\ \sum_{s=1}^{t} \overline{\omega}^{\varepsilon_s}_{s,a}(\widehat{\phi}_s) - n_t(s,a)
\label{eq:sampling_rule}
\end{equation}
with ties broken arbitrarily. The projection onto $\Sigma^\varepsilon$ forces a minimal amount of exploration so that no pair is left under-explored because of bad initial estimates. The same analysis of the sampling rule given in \cite{garivier16a} holds in the MDP case and guarantees that: 
\begin{equation*}
\PP_{\phi}\left(\forall (s,a) \in \scal\times\acal, \quad \underset{t \to \infty}{\lim} \frac{n_t(s,a)}{t} = \overline{\omega}_{s,a}(\phi)  \right) = 1.     
\end{equation*}

\subsection{Stopping rule}\label{subsection:test}

It is first worth noting that the proposed stopping condition constitutes the first stopping rule for best-policy identification in the MDP setting. Previous stopping rules in the literature are designed to identify $\varepsilon$-optimal policies. Unless we have access to an oracle that reveals the minimal gap between the best policy and a sub-optimal policy (in which case we can set $\varepsilon$ smaller than this gap), we cannot identify the best-policy using these rules.\\ 
A {\it good} stopping rule determines when the set of samples collected so far is {\it just} enough to declare that $\widehat{\pi}_t^\star = \pi^\star$ with probability $1-\delta$. The design of our stopping rule is inspired by the proof of the upper-bound $U(\phi)$, which uses the following fact (refer to the inequalities (\ref{ineq:1})-(\ref{ineq:2})-(\ref{ineq:3})-(\ref{ineq:4})-(\ref{ineq:5}) in the appendix): For all $\psi \in \Alt\phi$, there exists $s,a\neq \pi^{\star}(s)$ and a vector $\alpha$ in the simplex of $\mathbbm{R}^4$ (which we denote $\Sigma_4$) such that the four following conditions are verified: 
\begin{equation}
\begin{cases}
     \frac{\alpha_1^2}{T_1(s,a;\phi)}\leq \kl\left(r_{\phi}(s,a),r_{\psi}(s,a)\right), \\
     \frac{\alpha_2^2}{T_2(s,a;\phi)} \leq KL\left(p_{\phi}(s,a)\| p_{\psi}(s,a)\right), \\
    \frac{\alpha_3^2}{T_3(\phi)} \leq \ \underset{s \in \scal}{\max}\ \kl\left(r_{\phi}(s,\pi_{\phi}^\star(s)), r_{\psi}(s,\pi_{\phi}^\star(s))\right), \\ 
     \frac{\alpha_4^2}{T_4(\phi)} \leq \ \underset{s \in \scal}{\max}\ KL\left(p_{\phi}(s,\pi_{\phi}^\star(s))\| p_{\psi}(s,\pi_{\phi}^\star(s))\right).
\end{cases}
\label{eq:stopping_rule_1}
\end{equation}
Then defining the quantities
\begin{equation}
\begin{cases}
\rho_1(\phi,\psi)(s,a) = T_1(s,a;\phi) \kl\left(r_{\phi}(s,a),r_{\psi}(s,a)\right), \\
\rho_2(\phi,\psi)(s,a) =  T_2(s,a;\phi) KL\left(p_{\phi}(s,a)\| p_{\psi}(s,a)\right), \\
\rho_3(\phi,\psi) = \underset{s \in \scal}{\max}\ T_3(\phi) \kl (r_{\phi}(s,\pi_{\phi}^\star(s)), r_{\psi}(s,\pi_{\phi}^\star(s))), \\
\rho_4(\phi,\psi) = \underset{s \in \scal}{\max}\ T_4(\phi) KL(p_{\phi}(s,\pi_{\phi}^\star(s))\| p_{\psi}(s,\pi_{\phi}^\star(s))),
\end{cases}
\label{eq:rho}
\end{equation}
(\ref{eq:stopping_rule_1}) suggests that to design a PAC stopping condition, it is sufficient to check that the event
\begin{equation*}
\begin{split}
 \mathcal{E} = \bigg(& \forall \alpha \in \Sigma_4 \ \forall s,a\neq \widehat{\pi}_{t}^{\star}(s),\  \rho_1(\widehat{\phi}_t,\phi)(s,a) < \alpha_1^2 \textrm{ or }  \rho_2(\widehat{\phi}_t,\phi)(s,a) < \alpha_2^2\\
 &\textrm{ or } \rho_3(\widehat{\phi}_t,\phi) < \alpha_3^2 \textrm{ or } \rho_4(\widehat{\phi}_t,\phi) < \alpha_4^2 \bigg)
 \end{split}
\end{equation*} 
or equivalently\footnote{Hence the name KL-Ball stopping rule.}:
\begin{equation*}
\begin{split}
 \mathcal{E} = \bigg(\forall s,a\neq \widehat{\pi}_{t}^{\star}(s), 
 \sqrt{\rho_1(\widehat{\phi}_t,\phi)(s,a)}+ \sqrt{\rho_2(\widehat{\phi}_t,\phi)(s,a)}+\sqrt{\rho_3(\widehat{\phi}_t,\phi)}+\sqrt{\rho_4(\widehat{\phi}_t,\phi)} < 1\bigg)
 \end{split}
\end{equation*} 
holds with probability $1-\delta$. Indeed, if $\mathcal{E}$ holds, then by contraposition of (\ref{eq:stopping_rule_1}), we have $\phi \notin \mathrm{Alt}(\widehat{\phi}_t)$, which means that $\widehat{\pi}_t^\star = \pi^\star$. To define our stopping rule, we further introduce the threshold function:
\begin{equation*}
x(\delta,n,m) = \log(1/\delta) + (m-1)[1+\log\big(1+n/(m-1)\big)].
\end{equation*}
We finally define $\widehat{T}_1(s,a)=T_1(s,a;\widehat{\phi}_t)$, $\widehat{T}_2(s,a)=T_2(s,a;\widehat{\phi}_t)$, $\widehat{T}_3=T_3(\widehat{\phi}_t)$, $\widehat{T}_4=T_4(\widehat{\phi}_t)$ and $\delta' = \frac{\delta}{4S^3A}$. The KL-Ball stopping condition, which guarantees that the event $\mathcal{E}$ above holds with probability $1-\delta$, is:  
\begin{align}
&\underset{s,a\neq \widehat{\pi}_{t}^{\star}(s)}{\max} \frac{\sqrt{\widehat{T_1}(s,a) x(\delta',n_t(s,a),2)}+\sqrt{\widehat{T_2}(s,a) x(\delta',n_t(s,a),S)}}{\sqrt{n_t(s,a)}} \nonumber\\
&+\underset{s\in \scal}{\max} \frac{\sqrt{\widehat{T_3} x(\delta',n_t(s,\widehat{\pi}_t^\star(s)),2)  }+ \sqrt{\widehat{T_4} x(\delta',n_t(s,\widehat{\pi}_t^\star(s)),S)}}{\sqrt{n_t(s,\widehat{\pi}_t^\star(s))}} \leq 1
\label{eq:simplifed_stopping_rule}
\end{align}
More precisely: $\tau_ {\delta}=\inf\{ t\in\mathbb{N}: (\ref{eq:simplifed_stopping_rule}) \textrm{ holds}\}$.

\begin{theorem} Under the KL-Ball stopping rule, we have: $\PP_{\phi}(\tau_\delta < \infty,\widehat{\pi}_{\tau_\delta}^\star \neq \pi^\star_{\phi} ) \leq \delta$.
\label{Theorem:stopping_rule}
\end{theorem}

\section{SAMPLE COMPLEXITY ANALYSIS}\label{sec:analysis}

Our main results take the form of asymptotic (when $\delta$ goes to 0) upper bounds on the sample complexity of KLB-TS. These bounds are proved as follows. First, the use of the C-tracking rule makes it possible to establish the convergence of the vector $(n_t(s,a))_{s,a}/t$ (the (state, action) pair visit frequencies) to the nearly-optimal allocation vector $\overline{\omega}$, as well as the convergence of the empirical MDP $\widehat{\phi}_t$ to the true MDP $\phi$. Then, plugging these convergence results in the definition of the stopping rule (\ref{eq:simplifed_stopping_rule}), and combining the obtained results with the asymptotic shape of the threshold function $x(\delta',n,m) \underset{\delta \to 0}{\sim} \log(1/\delta)$, we obtain (refer to Appendix \ref{AppendixC} for a detailed description of these arguments):
\begin{align*}
\tau_{\delta} \underset{\delta \to 0}{\sim} \inf\Bigg\{ t\in\mathbb{N}:\sqrt{\log(1/\delta)}\bigg(&\underset{s,a\neq \pi^\star (s)}{\max} +\frac{\sqrt{T_1(s,a;\phi)}+ \sqrt{T_2(s,a;\phi) }}{\sqrt{t\times \overline{\omega}_{sa}}}\\ 
&+ \underset{s\in \scal}{\max} \frac{\sqrt{T_3(\phi)}+ \sqrt{T_4(\phi)}}{\sqrt{t\times \overline{\omega}_{s,\pi^\star(s)}}} \ \bigg) \leq 1 \Bigg\}.
\end{align*}
Finally, we show that the condition in the '$\inf$' above holds as soon as $t \geq 4U(\phi) \log(1/\delta)$ (see Lemma \ref{lemma:technical_bound}). The above arguments lead to an upper bound of the sample complexity of KLB-TS, valid almost surely (Proposition \ref{proposition:almost_sure_sample_complexity}) and in expectation (Theorem \ref{Theorem:Upper bound in expectation}).

\begin{proposition}\label{proposition:almost_sure_sample_complexity}
The KL-Ball stopping rule, coupled with any sampling rule ensuring that for every state-action pair $(s,a)$, $n_t(s,a)/t$ converges almost surely to the nearly-optimal allocations $\overline{\omega}_{s,a}$ of  Corollary \ref{corollary:upper_bound}, yields a sample complexity $\tau_\delta$ satisfying for all $ \delta \in (0,1): \PP_{\phi}(\tau_{\delta} < \infty) = 1 $ and
$\PP_{\phi}\left(\limsup_{\delta \to 0} \frac{\tau_{\delta}}{\log(1/\delta)} \leq 4 U(\phi)\right) = 1.$
\end{proposition}

\begin{theorem}
The KL-Ball stopping rule, coupled with the C-tracking rule defined in (\ref{eq:sampling_rule}), yields a sample complexity $\tau_\delta$ satisfying: for all $\delta \in (0,1),\ \EE_{\phi}[\tau_{\delta}]$ is finite and $\limsup_{\delta \to 0} \frac{\EE_{\phi}[\tau_{\delta}]}{\log(1/\delta)} \leq 4 U(\phi).$
\label{Theorem:Upper bound in expectation}
\end{theorem}
The proof of the theorem above is similar to that of Theorem 14 in \cite{garivier16a} with a few notable differences. First, we defined a distance on MDPs through the $L^{\infty}$-norm of their reward and transition kernels. Then, we adapted Lemma 19 from \cite{garivier16a}, which gives a concentration inequality of the empirical average-rewards in the MAB setting, to include the concentration of transition probabilities of the empirical MDP.

\section{EXPERIMENTS}\label{sec:experiments}

In this section, we run numerical experiments to compare the performances of KLB-TS and BESPOKE (these are so far the two algorithms with problem-specific sample complexity guarantees). We refer the reader to Appendix \ref{sec:compare} for a detailed description of the differences between KLB-TS and BESPOKE, as well as a comparison of their theoretical guarantees. To compare the two algorithms, we generated two MDPs randomly: a first small MDP with two states and two actions, and a second larger and more realistic MDP with five states and ten actions per state. We used BESPOKE with an accuracy parameter $\epsilon = 0.9\Delta_{\min}$ (note that $\Delta_{\min}$ is revealed to BESPOKE). For each value of the confidence level $\delta$, we run 10 simulations for the first MDP under both algorithms. To save computation time in the case of the second MDP, we run 5 simulations for each $\delta$ and only compare KLB-TS's sample complexity with BESPOKE's initial number of samples $n_{\min}$ which, as noted in Appendix \ref{sec:compare}, contributed for more than 99\% of its sample complexity. 

Figure \ref{fig:experiment1} shows the mean sample complexity along with its 2-standard-deviations interval (which seems very small due to the use of a log-scale). The red curve (referred to as 'asymptotic bound') shows the upper bound $4 U(\phi) \log(1/\delta)$ guaranteed by Theorem \ref{Theorem:Upper bound in expectation}. Note that KLB-TS sample complexity is greater than $4 U(\phi) \log(1/\delta)$ for moderate values of $\delta$ and only matches it for $ \delta = 10^{-14}$. For both MDPs, KLB-TS clearly outperforms BESPOKE.

\begin{figure*}[h!]
     \centering
         \includegraphics[width =0.32\linewidth]{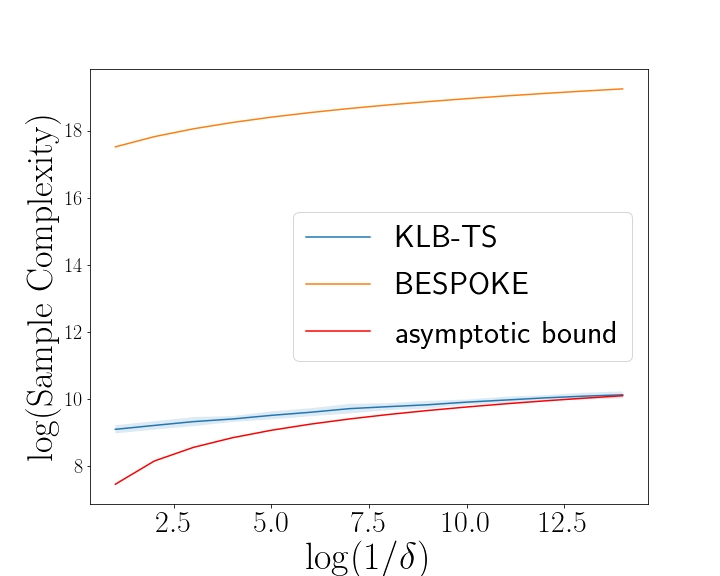}
          \includegraphics[width =0.32\textwidth]{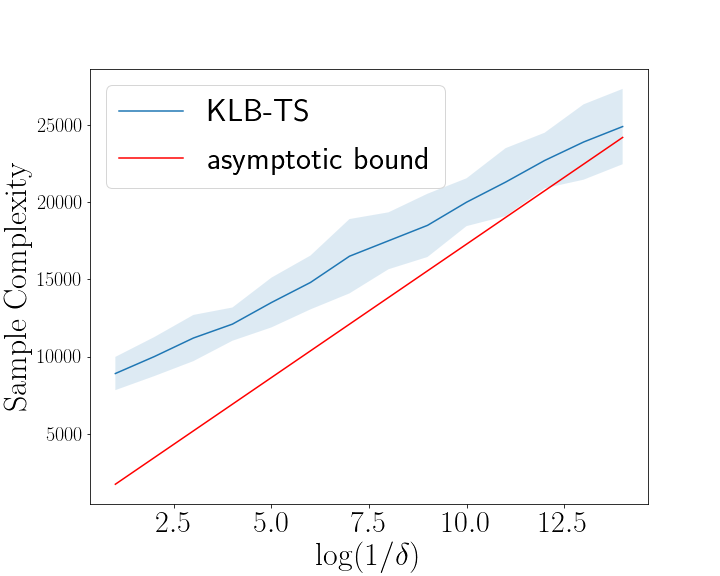}
         \includegraphics[width=0.32\textwidth]{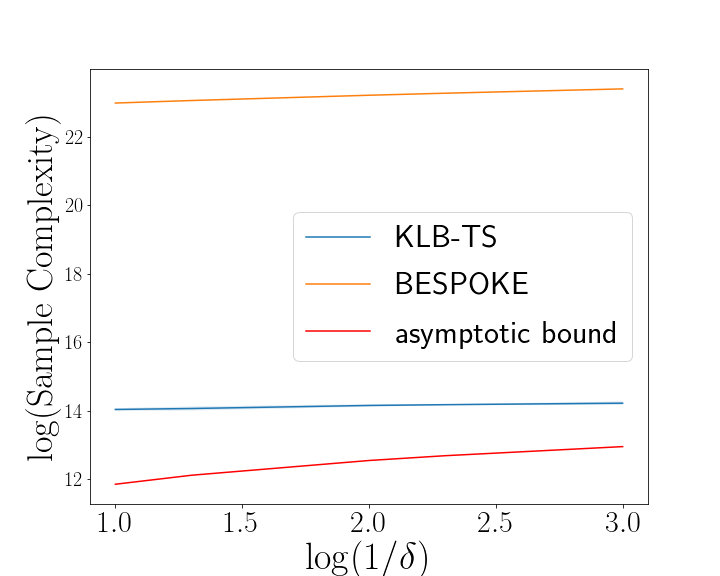}
\caption{KLB-TS vs. BESPOKE. Left and center: S=A=2, $\gamma = 0.5$, right: $S=5, A=10, \gamma=0.7$.}
\label{fig:experiment1}
\end{figure*}


\section{CONCLUSION}\label{sec:conclusion}

In this work, we have investigated the design of RL algorithms with {\it minimal problem-specific} sample complexity. To this aim, we first derived the information-theoretical sample complexity limit (a lower bound on the sample complexity satisfied by any algorithm) and the corresponding optimal sample allocation. Our hope was that, as for the MAB problem, this allocation would be easy to compute and could then lead to a simple and optimal track-and-stop algorithm. Unfortunately, for RL problems, it turns out that the optimal allocation solves an involved non-convex program. Approaching the fundamental sample complexity limit seems possible only if one could solve this program. To circumvent this issue, we derived a tight upper bound of the characteristic time. Remarkably, this bound corresponds to a sample allocation that is explicit, and hence can be easily plugged in into a track-and-stop algorithm. Based on this upper bound, we proposed KLB-TS, an algorithm whose sample complexity matches this upper bound. 

This work opens up interesting research directions. First, the computational complexity of the sample complexity lower bound strongly suggests the existence of a fundamental trade-off between sample and computational complexities. Investigating this trade-off is intriguing. Then, we restricted our attention to the generative model, where one can sample any (state, action) pair at any step. In most practical cases however, one needs to learn an optimal policy by observing a single trajectory of the system. Hence, the numbers of times one observes the various (state, action) pairs are correlated, inducing some additional constraints in the optimization problem leading to the sample complexity lower bound. It is worth studying the impact of these navigation constraints on the sample complexity. Finally, we plan to extend our results to the framework of RL with function approximation.  
\clearpage
\newpage
\bibliography{bibliography,ref,vr2017}

\appendix
\onecolumn

\section{Related work: The minimax approach}

One of the first works on best-policy identification in discounted MDPs is \cite{Kearns1999}. There, the authors introduce a model referred to as {\it parallel sampling}, where the agent can sample transitions from all (state,action) pairs simultaneously (instead of following a trajectory determined by the MDP dynamics). They proposed Phased Q-Learning and showed that it requires $\Tilde{\mathcal{O}}\left(\frac{{S}{A}\log({S}{A}/\delta)}{\varepsilon^2 }\right)$ samples\footnote{Their analysis ignored the dependency on the horizon $H = \frac{1}{1-\gamma}$, treating $\gamma$ as a constant.} to find an $\varepsilon$-optimal value function.  Later on, \cite{kakade2003sample}(Chapter~2.5) proposed the generative model as a variant of the parallel sampling model. Both \cite{Kearns1999} and  \cite{kakade2003sample} proved upper-bounds on the sample complexity of model-based Q-Value-Iteration (QVI) by  $\mathcal{O}\left(\frac{{S}{A}\log({S}{A}/\delta)}{\varepsilon^2 (1-\gamma)^4}\right)$. Using a variance trick, \cite{azar2013minimax} improved their analysis and showed that when $\varepsilon \in (0,\frac{1}{\sqrt{(1-\gamma)S}}]$, both model-based QVI along and Policy Iteration (PI) can find an $\varepsilon$-optimal policy after collecting $\mathcal{O}\left(\frac{S{A}\log(SA/\delta)}{\varepsilon^2 (1-\gamma)^3}\right)$ samples. They also proved that the latter quantity is the minimax lower bound of sample complexity required to find an $\varepsilon$-optimal policy. \cite{even2006action} used Action-Elimination techniques from the Multi-Armed Bandit setting(MAB) to devise MAB-Phased-Q-Learning, an algorithm for MDPs with a generative model which finds an $\varepsilon$-optimal policy using $\Tilde{\ocal}(\frac{S A V_{\max}^2}{(1-\gamma)^5\varepsilon^2})$ samples, where $V_{\max}$ is the maximum range of the value function. \cite{NIPS2018_7765} proposed Variance-Reduced-Q-Value-Iteration (vQVI) which matches the minimax bound for a wider range of $\varepsilon \in (0,1]$. The same bound was derived by \cite{pmlr-v125-agarwal20b} for $\varepsilon \in (0, \frac{1}{\sqrt{1-\gamma}}]$ using a model-based approach. Finally, \cite{li2020breaking} used a reward perturbation technique to widen the set of $\varepsilon$ where their algorithm is minimax optimal to the full range of accuracy levels: $(0,\frac{1}{1-\gamma}]$. It is worth noting that, except for \cite{even2006action}, the aforementioned papers only sample transitions and assume a reward function known in advance by the agent.


\section{Additional Proprerties of the lower bound program}\label{sec:appProp}
{\bf Most alternative MDPs.} We refer to an MDP $\psi\in \overline{\Alt\phi}$\footnote{We use $\overline{E}$ to denote the closure of a set $E$.} solving the problem (\ref{eq:subprob}) as {\it most alternative}, since for a given allocation $\omega$, the sample complexity lower bound is determined by the number of samples needed to distinguish $\phi$ from $\psi$.

Observe that the condition (\ref{eq:rewritten_conditions_with_kernels}) involves transition probabilities and rewards of the (state, action) pairs $(s,a)$ and $(s',\pi^\star(s'))$ for all $s'$, only. Hence $\psi\in \mathrm{Alt}_{sa}(\phi)$ can be obtained from $\phi$ by changing at most the transition probabilities and rewards of these (state, action) pairs. Next, let $\psi\in \overline{\mathrm{Alt}_{sa}(\phi)}$ solve (\ref{eq:subprob}). Then we can verify that the constraint (\ref{eq:rewritten_conditions_with_kernels}) is active and that we have:
\begin{align}
\Delta_{sa}  = dr(s,a)  + \gamma dp(s,a)^{\top} V_{\phi}^{\pi^\star} + [\gamma p_{\psi}(s,a) - \mathbbm{1}(s)] ^{\top}dV^{\pi^\star}. \nonumber
\end{align}
This means that to design a most alternative MDP, one should change the rewards and transitions of optimal (state, action) pairs and only one sub-optimal pair $(s,a)$ and those changes should be just enough to fill sub-optimality gap $\Delta_{sa}$. The next lemma formalizes these findings.

\begin{lemma} Denote by $\ocal(\phi) = \{(s,a) :\ \starQ{\phi}(s,a) = \starV{\phi}(s) \}$ the set of optimal (state,action) pairs in the MDP $\phi$ and let $\psi\in \overline{\Alt\phi}$ solve (\ref{eq:subprob}). Then: \\
(i) For all $(s,a)\in \scal \times \acal$, $\left(p_{\psi}(.|s,a), q_{\psi}(.|s,a) \right) \neq \left(p_{\phi}(.|s,a), q_{\phi}(.|s,a) \right)  \ \implies \ (s,a) \in \ocal(\psi)\setminus \ocal(\phi)$ or $a=\pi^\star(s)$;\\
(ii) $\ocal(\phi) \subset \ocal(\psi)$.
\label{lemma:3}
\end{lemma}
\begin{proof}
First we recall the following facts which we  will make  use of.

{\bf Fact 1.} $\starQ{}$ is Liptschitz w.r.t rewards and transitions (by simple bounds on Bellman operator): 

\begin{equation*}
 \norm{Q_{\phi}^\star-Q_{\psi}^\star}_{\infty} \leq \left(1+\frac{1}{1-\gamma} \right) \left( \norm{r_{\phi}-r_{\psi}}_{\infty} +  \frac{\gamma}{(1-\gamma)}\norm{p_{\phi}-p_{\psi}}_{1,\infty} \right).  
\end{equation*}

{\bf Fact 2.} If we change only the kernels $\left( p_{\phi}(s,a),q_{\phi}(s,a) \right) \to \left( p_{\psi}(s,a),q_{\psi}(s,a) \right)$ of some sub-optimal (state, action) pair $s, a\neq \pi^\star(s)$  and the action $a$ doesn't become strictly optimal $(s,a)\notin \ocal(\psi)$, then the value function remains unchanged $\starV{\psi} = \starV{\phi}$.

This is because there exists $(\pi_1,\pi_2) \in \Pi_{\phi}^\star \times \Pi_{\psi}^\star $ such that $ \pi_2(a|s) = \pi_1(a|s) = 0$ (where we recall that $\pi(a|s)$ denotes the probability that $\pi$ selects $a$ in state $s$) which implies: 
\begin{equation*}
\begin{cases}
\left( P_{\psi}^{\pi_1},r_{\psi}^{\pi_1} \right) = \left( P_{\phi}^{\pi_1},r_{\phi}^{\pi_1} \right) \\
\\
\left( P_{\psi}^{\pi_2},r_{\psi}^{\pi_2} \right) = \left( P_{\phi}^{\pi_2},r_{\phi}^{\pi_2} \right)
\end{cases} 
\implies
\begin{cases}
  \starV{\psi} \geq V_{\psi}^{\pi_1} = \left(I - \gamma P_{\psi}^{\pi_1}\right)^{-1}r_{\psi}^{\pi_1}  = \left(I - \gamma P_{\phi}^{\pi_1}\right)^{-1}r_{\phi}^{\pi_1}  = \starV{\phi} \\
  \\
  \starV{\phi} \geq V_{\phi}^{\pi_2} = \left(I - \gamma P_{\phi}^{\pi_2}\right)^{-1}r_{\phi}^{\pi_2}  = \left(I - \gamma P_{\psi}^{\pi_2}\right)^{-1}r_{\psi}^{\pi_1}  = \starV{\psi} 
\end{cases}
\end{equation*}

\textbf{Fact 3:}  We can restrict our attention to allocation vectors $\omega$ with zero-null entries: $\forall (s,a) \in \scal \times \acal: \ \omega_{s a} > 0$. 

In fact, any allocation vector $\omega$ such that $\omega_{s a} = 0$ is suboptimal. Indeed, consider $\psi$ obtained from $\phi$ by changing the kernels in $(s,a)$ so that they become equal to the kernels in $(s,\pi^\star(s))$, while keeping everything else unchanged. Then by definition of $\psi$: $\underset{s',a'}{\sum}\omega_{s',a'}KL_{\phi|\psi}(s',a') = 0$. Furthermore one can easily show that $\psi \in \overline{\Alt\phi}$ which implies that $K(\phi,\omega)^{-1} = 0$.

\medskip
We are now ready to prove the lemma. Let $\psi \in \overline{\Alt\phi}$ solving (\ref{eq:subprob}). We can write:  $\psi = \underset{n \to \infty}{\lim} \psi_n $,  where $({\psi_n})_{n \geq 1} \in \Alt\phi^{\mathbb{N}} $ and $ \underset{n \to \infty}{\lim} \sum_{s,a}\omega_{s a}KL_{\phi|\psi_n}(s,a) = \underset{\psi \in \Alt\phi}{\inf}\sum_{s,a}\omega_{s a}KL_{\phi|\psi}(s,a)$. Therefore, by continuity of the KL function:
\begin{equation}
 \sum_{s,a}\omega_{s a}KL_{\phi|\psi}(s,a) =  \underset{\psi \in \Alt\phi}{\inf}\sum_{s,a}\omega_{s a}KL_{\phi|\psi}(s,a) 
 \label{eq:optimality_of_psi}
 \end{equation}
 
\paragraph{ \underline{Proof of (i): $\left(p_{\psi}(.|s,a), q_{\psi}(.|s,a) \right) \neq \left(p_{\phi}(.|s,a), q_{\phi}(.|s,a) \right) \implies  (s,a)\in \ocal(\psi)\setminus\ocal(\phi) \textrm{ or } a = \pi^\star(s) $}\\ \\} 

By contradiction:
Suppose there exists $(s,a)$ such that:  $ \left( p_{\psi}(s,a),q_{\psi}(s,a) \right) \neq \left(p_{\phi}(s,a),q_{\phi}(s,a)\right)$ and $(s,a) \in \ocal(\psi)^{c} \cup \ocal(\phi)$ and $a \neq \pi^\star(s)$. Combined together, the latter two conditions imply that:
\begin{equation}
(s,a) \in \ocal(\psi)^{c}.
\label{eq:ocal_proof_lemma_3}
\end{equation}
We will use the following operator ($\varepsilon$-transform) where we move the rewards and transitions of $\psi$ at $(s,a)$ in the direction of $\phi$ by $\varepsilon \geq 0$: $T_{\phi, \varepsilon}^{s,a}(\psi) \triangleq \psi_{\varepsilon}$ where 
\begin{equation}
 \left( p_{\psi_{\varepsilon}}(s',a'),q_{\psi_{\varepsilon}}(s',a') \right) =   \begin{cases}
    (1-\varepsilon)\left( p_{\psi}(s,a),q_{\psi}(s,a) \right) + \varepsilon \left( p_{\phi}(s,a),q_{\phi}(s,a) \right),  \quad \textrm{if $(s',a') = (s,a)$}, \\
     \left( p_{\psi}(s',a'),q_{\psi}(s',a') \right) \quad \textrm{otherwise.}
    \end{cases}
\label{eq:epsilonTransform}
\end{equation}
Note that the objective function of the infimum problem takes a smaller value at $\psi_\varepsilon$ than at $\psi$: 
\begin{equation*}
\begin{split}
\sum_{s',a'}\omega_{s',a'}KL_{\phi|\psi_{\varepsilon}}(s',a') &\leq \left[(1-\varepsilon)\ \omega_{s a} KL_{\phi|\psi}(s,a) + \varepsilon \ \omega_{s a} KL_{\phi|\phi}(s,a)\right] + \sum_{(s',a') \neq (s,a)}\omega_{s',a'} KL_{\phi|\psi}(s',a') \\
&< \sum_{s',a'}\omega_{s',a'}KL_{\phi|\psi}(s',a')
\end{split}
\end{equation*}
where the first inequality stems from the convexity of KL-function and the second from the property $p\neq q \implies KL(p\|q) > 0$. We will prove that there exists $\varepsilon >0$ such that $\psi_\varepsilon$ is the limit of a sequence of elements in $\Alt\phi$, which clearly contradicts the optimality of $\psi$ (see equation \ref{eq:optimality_of_psi}).

Consider $a^\star$ an optimal action at state $s$ in $\psi$, ie such $(s,a^\star) \in \ocal(\psi) $. Since $(s,a) \notin \ocal(\psi)$ (\ref{eq:ocal_proof_lemma_3}), then for $\varepsilon = 0$, we have: $\psi_{0} = \psi$ and $ \delta \triangleq \delta_{\psi}(s,a) =  \starQ{\psi}(s,a^\star) - \starQ{\psi}(s,a) > 0 $. By continuity of $\starQ{}$ w.r.t the rewards and transitions (\textbf{Fact 1}), there exists $\varepsilon > 0$ small enough such that:
\begin{equation*}
      \starQ{\psi_{\varepsilon}}(s,a^\star) - \starQ{\psi_{\varepsilon}}(s,a) > \delta/2 > 0.
\end{equation*}
Fix such $\varepsilon$ and define $(\theta_n)_{n \geq 1} = \left( T_{\phi, \varepsilon}^{s,a}(\psi_n) \right)_{n \geq 1}$ where $(\psi_n)_{n \geq 1}$ is any sequence converging to $\psi$. By continuity of the operator $T_{\phi, \varepsilon}^{s,a}$, we have: $\underset{n \to \infty}{\lim} \theta_n = \psi_{\varepsilon}$. It remains to show that $(\theta_n)_{n \geq 1} \in \Alt\phi^{\mathbbm{N}}$. Using the continuity of $\starQ{}$ another time, we get:
 \begin{equation*}
     \begin{split}
       \begin{cases}
     \underset{n \to \infty}{\lim} \psi_n = \psi \\
    \underset{n \to \infty}{\lim} \theta_n = \psi_{\varepsilon} 
    \end{cases}
       & \implies \begin{cases}
       \underset{n \to \infty}{\lim} \starQ{\psi_n}(s,a^\star) - \starQ{\psi_n}(s,a) =  \starQ{\psi}(s,a^\star) - \starQ{\psi}(s,a) >  \delta/2 \\
       \underset{n \to \infty}{\lim} \starQ{\theta_n}(s,a^\star) - \starQ{\theta_n}(s,a) =  \starQ{\psi_{\varepsilon}}(s,a^\star) - \starQ{\psi_{\varepsilon}}(s,a) >  \delta/2
       \end{cases} \\ 
       & \implies \exists N_0 \in \mathbb{N} \quad \forall n \geq N_0 \quad
      \begin{cases}
      \starQ{\psi_n}(s,a^\star) - \starQ{\psi_n}(s,a) > \delta/2 \\
      \starQ{\theta_n}(s,a^\star) - \starQ{\theta_n}(s,a) > \delta/2 \\
      \end{cases} \\
      & \implies \forall n \geq N_0 \:  \textrm{(s,a) is sub-optimal in both $\psi_n$ and $\theta_n$.  }
     \end{split}
 \end{equation*}
This implies, by Fact 2 on $\psi_n$ and $\theta_n$, that: $\forall n \geq N_0 \ \starV{\theta_n} = \starV{\psi_n}$. Since, we only changed kernels of $\psi_n$ at $(s,a)$ to obtain $\theta_n$, then this also implies that for all $n \geq N_0$:  
\begin{equation*}
\begin{cases}
\forall (s',a') \neq (s,a),\ \starQ{\psi_n}(s',a') = r_{\psi_n}(s',a') + \gamma p_{\psi_n}(s',a')^{T}\starV{\psi_n} = r_{\theta_n}(s',a') + \gamma p_{\theta_n}(s',a')^{T}\starV{\theta_n} = \starQ{\theta_n}(s',a')\\
\textrm{ (s,a) is sub-optimal in both $\psi_n$ and $\theta_n$}
\end{cases}
\end{equation*}

Therefore, $\forall n \geq N_0,\ \Pi_{\theta_n}^{*} = \Pi_{\psi_n}^{*}$, and consequently $\theta_n \in \Alt\phi$.\\
To sum up, modulo a reindexing of the sequence: $ \exists (\theta_n)_{n \geq 1} \in \Alt\phi^{\mathbb{N}}:  \underset{n \to \infty}{\lim} \theta_n = \psi_{\varepsilon}$. This is a contradiction.

\paragraph{ \underline{Proof of (ii): $ \ocal(\phi) \subset \ocal(\psi) $}\\ \\}
We proceed in the same way, i.e., we suppose that there exists $(s,a)\in \ocal(\phi)\setminus\ocal(\psi)$. Only this time, we consider $ \psi_{\varepsilon} \triangleq \underset{s',a'}{\prod}\ T_{\phi, \varepsilon}^{s',a'}(\psi)$ where the product sign stands for composition of operators. It's straightforward to show, using continuity of $\starQ{}$ w.r.t rewards and transitions, that there exists $\varepsilon > 0$ such that $(s,a)$ is still not optimal: $a \notin \ocal(\psi_\varepsilon)$. Hence $\psi_\varepsilon \in \Alt\phi$, which contradicts the optimality of $\psi$.
\end{proof}

\clearpage
\newpage

\section{Lower Bound $T^*(\phi)$}
\label{AppendixA} 
\subsection{Proof of Lemma \ref{lemma:wald}}
\begin{proof}
Let $\tau$ be a stopping time w.r.t. the filtration $({\cal F}_t)_{t\ge 1}$. The observations made up to the beginning of round $t$ are ${\cal O}_t=(s_1,a_1,R_1,s'_1\ldots, s_{t},a_{t},R_{t},s'_{t})$. Let $p(\cdot)$ denote the distribution of the first state. We have:
\begin{align*}
\mathbb{P}_\phi({\cal O}_t)=p(s_1)\prod_{k=1}^t  p_\phi(s'_{k}|s_{k},a_{k}) \times \prod_{k=1}^{t}q_\phi(R_k|s_k,a_k).
\end{align*}
The log-likelihood ratio of the observations up to the end of round $t$ under $\phi$ and $\psi$ is then:
\begin{align*}
L_t & = \sum_{k=1}^t \left(\log{p_\phi(s'_{k}|s_{k},a_{k})\over p_\psi(s'_{k}|s_{k},a_{k})} +\log{q_\phi(R_{k}|s_{s},a_{k})\over q_\psi(R_{k}|s_{k},a_{k})}\right)\\
& = \sum_{s,a} L_t^{s,a},
\end{align*}
where 
\begin{align*}
L_t^{s,a} = \sum_{k=1}^t & \mathbbm{1}_{\{ s_{k}=s,a_{k}=a\} }\left(  \log{p_\phi(s'_{k}|s,a)\over p_\psi(s'_{k}|s,a)} +\log{q_\phi(R_{k}|s,a)\over q_\psi(R_{k}|s,a)}\right).
\end{align*}
Next we study $L_t^{s,a}$ for a given pair $(s,a)$. Introduce the following random variables: $Y_k$ and $Z_k$ denote the next state and the collected reward after the $k$-th time $(s,a)$ has been visited. We can re-write $L_t^{s,a}$ as:
$$
L_t^{s,a} = \sum_{k=1}^{N_t(s,a)}\left( \log{p_\phi(Y_k|s,a)\over p_\psi(Y_k|s,a)} + \log{q_\phi(Z_{k}|s,a)\over q_\psi(Z_{k}|s,a)}\right)
$$
Observe that $\xi_k:= \log{p_\phi(Y_k|s,a)\over p_\psi(Y_k|s,a)} + \log{q_\phi(Z_{k}|s,a)\over q_\psi(Z_{k}|s,a)}$ and $\mathbbm{1}_{\{ N_\tau(s,a)>k-1\} }$ are independent, because under the event  $\{N_\tau(s,a)\le k-1\}$, $Y_s$ and $Z_s$ have not been observed yet. Further notice that $\mathbb{E}_\phi[\xi_k]=\KL_{\psi \mid \phi} (s,a)$. We deduce that:
\begin{align*}
\mathbb{E}_\phi[L_\tau^{s,a}] & = \mathbb{E}_\phi\left[ \sum_{k=1}^\infty \xi_k \mathbbm{1}_{\{ N_\tau(s,a)>k-1\} }\right] \\
& = \sum_{k=1}^\infty\mathbb{P}_\phi[N_\tau(s,a)>k-1]\KL_{\psi \mid \phi} (s,a)\\
& = \mathbb{E}_\phi[N_\tau(s,a)]\KL_{\psi \mid \phi} (s,a).
\end{align*}
Summing over all pairs $(s,a)$ completes the proof.
\end{proof}

\clearpage
\newpage

\section{Main properties of the problem (\ref{Prob:main})}
\subsection{Proof of Lemma \ref{lemma:2}}
\begin{proof}
To simplify the notation, we denote $\pi = \pi_{\phi}^\star$.\\[0.7cm]
\textbf{\underline{ First part: $\Alt\phi \subset \underset{s,a\neq \pi^\star(s)}{\bigcup}\{\psi:  Q_{\psi}^{\pi}(s,a) > V_{\psi}^{\pi}(s) \}$}  }\\ \\
By contradiction: Suppose there exists $\psi \in \Alt\phi$ such that $\forall s,a\neq \pi^\star(s),\ Q_{\psi}^{\pi}(s,a) \leq V_{\psi}^{\pi}(s)$. Since $Q_{\psi}^{\pi}(s,\pi(s)) = V_{\psi}^{\pi}(s)$ then the inequality is valid for all pairs: 
$$\forall (s,a)\in \scal\times\acal,\ Q_{\psi}^{\pi}(s,a) \leq V_{\psi}^{\pi}(s)$$

Let $\pi_{\psi}^\star$ be an optimal policy under $\psi$. Then:  $$\forall s \in \scal,\ Q_{\psi}^{\pi}(s,,\pi_{\psi}^\star(s)) \leq V_{\psi}^{\pi}(s) $$
Define the Bellman operator of $\pi$ under $\psi$ as $\bcal_{\psi}^\pi: \mathbb{R}^S\to \mathbb{R}^S$ and for all $s\in {\cal S}$,
$$
(\bcal_{\psi}^\pi V)(s) = r_\psi(s,\pi(s)) +\gamma p_\psi(s,\pi(s))^\top V.
$$
Using the Bellman operator of the policy $\pi_{\psi}^\star$ under $\psi$, we rewrite the inequalities above:
$$ \bcal_{\psi}^{\pi_{\psi}^\star} V_{\psi}^{\pi} \leq V_{\psi}^{\pi}. $$
By monotonicity of Bellman operator, this implies that: 
$\forall n \geq 1,\ \bigg(\bcal_{\psi}^{\pi_{\psi}^\star}\bigg)^{n}\ V_{\psi}^{\pi} \leq V_{\psi}^{\pi}.$ 
Hence: 
$$\starV{\psi} = \underset{n\to \infty}{\lim}\  \bigg(\bcal_{\psi}^{\pi_{\psi}^\star}\bigg)^{n}\ V_{\psi}^{\pi} \leq V_{\psi}^{\pi},$$
i.e., the policy $\pi$ is optimal under $\psi$. This is a contradiction.

\textbf{\underline{ Second part: $ \underset{s,a\neq \pi^\star(s)}{\bigcup}\{\psi:  Q_{\psi}^{\pi}(s,a) > V_{\psi}^{\pi}(s) \} \subset \Alt\phi$}  }\\ \\
By contradiction: Let $s,a\neq \pi^\star(s)$ and suppose there exists $\psi \in \{\psi:  Q_{\psi}^{\pi}(s,a) > V_{\psi}^{\pi}(s) \}$ such that $\pi = \pi_{\phi}^\star$ is optimal under $\psi$.
Define the modified policy $\pi_1$ as:
\begin{equation*}
    \pi_1(s') = \left\{
    \begin{array}{ll}
        a &\textrm{ if $s' = s$,}\\
        \pi(s') & \textrm{ otherwise.}
    \end{array}
    \right.
\end{equation*}
Then the fact that $Q_{\psi}^{\pi}(s,a) > V_{\psi}^{\pi}(s)$ translates to: $$\bcal_{\psi}^{\pi_1} \starV{\psi} = \bcal_{\psi}^{\pi_1} V_{\psi}^{\pi} > V_{\psi}^{\pi} = \starV{\psi} $$ 
where the equality comes from the assumption that $\pi$ is an optimal policy in $\psi$. Therefore, by monotonicity of Bellman operator, we have: 
$$V_{\psi}^{\pi_1} = \underset{n\to \infty}{\lim}\  \bigg(\bcal_{\psi}^{\pi_1}\bigg)^{n}\ \starV{\psi} > \starV{\psi}.$$ 
We got a a contradiction.
\end{proof}

\clearpage
\newpage

\section{Upper bound $U(\phi)$ and the near-optimal sampling allocation $\overline{\omega}$}
\label{sec:appendix_upper_bound}

\subsection{First technical lemma}

We will need the following technical lemma which relates the change in the future discounted rewards between $\phi$ and $\psi$ due to different transitions $dp(s,a)^{\top} 
\starV{\phi}$ to the Kullback-Leibler divergence of the transition kernels as well as the variance and maximum-deviation of the next-state value.
\begin{lemma}
Using the notations of Sections 4.1 and 4.2, we have:
\begin{equation*}
|dp(s,a)^{\top} 
\starV{\phi}|^2 \leq 8 \textrm{KL}(p_{\phi}(s,a)\| p_{\psi}(s,a)) \textrm{Var}_{p_{\phi}(s,a)}[\starV{\phi}] +4\sqrt{2} \textrm{KL}(p_{\phi}(s,a)\| p_{\psi}(s,a))^{3/2}\textrm{MD}_{p_{\phi}(s,a)}[\starV{\phi}]^2.
\end{equation*}
\label{lemma:hellinger}
\end{lemma}
\begin{proof}
We have: 
\begin{equation*}
\begin{split}
dp(s,a)^{\top} 
\starV{\phi} & = \sum_{s'} \left(p_{\psi}(s'|s,a) - p_{\phi}(s'|s,a)\right) \left[\starV{\phi}(s') - \EE_{\Tilde{s} \sim p_{\phi}(.|s,a) }[\starV{\phi}(\Tilde{s})]\right] \\
 = \sum_{s'} &\left(\sqrt{p_{\psi}(s'|s,a)} - \sqrt{p_{\phi}(s'|s,a)}\right)\\ 
 & \times \left[\left(\sqrt{p_{\psi}(s'|s,a)} + \sqrt{p_{\phi}(s'|s,a)}\right) \left(\starV{\phi}(s') - \EE_{\Tilde{s} \sim p_{\phi}(.|s,a) }[\starV{\phi}(\Tilde{s}))]\right) \right].
\end{split}
\end{equation*}

Thus, by Cauchy-Schwartz inequality:
\begin{equation*}
\begin{split}
|dp(s,a)^{\top} 
\starV{\phi}|^2  \leq & 2 d_{H}(p_{\phi}(s'|s,a),p_{\psi}(s'|s,a))^2 \times \\  
&\Bigg[\sum_{s'}\left(\sqrt{p_{\psi}(s'|s,a)} + \sqrt{p_{\phi}(s'|s,a)}\right)^2 \left(\starV{\phi}(s') - \EE_{p_{\phi}(.|s,a) }[\starV{\phi}(\Tilde{s}))]\right)^2 \Bigg]\\
& \leq 4 d_{H}(p_{\phi}(s'|s,a),p_{\psi}(s'|s,a))^2 \Bigg[\sum_{s'}\left(p_{\psi}(s'|s,a) + p_{\phi}(s'|s,a)\right) \left(\starV{\phi}(s') - \EE_{p_{\phi}(.|s,a) }[\starV{\phi}(\Tilde{s}))]\right)^2\Bigg],
\end{split}
\end{equation*}
where we have used $(a+b)^2 \leq 2 (a^2 + b^2) $ and $d_{H}(p,q)= \left[\frac{1}{2} \sum_{i} (\sqrt{p_i} -\sqrt{q_i})^2  \right]^{1/2}$ is the Hellinger distance between two probability distributions. Therefore:
\begin{equation*}
\begin{split}
|dp(s,a)^{\top} 
\starV{\phi}|^2  \leq & 4 d_{H}(p_{\phi}(s'|s,a),p_{\psi}(s'|s,a))^2\\ & \times \left[2\textrm{Var}_{s' \sim p_{\phi}(.|s,a)}[\starV{\phi}(s')] +\norm{p_{\phi}(s'|s,a)-p_{\psi}(s'|s,a)}_{1}\norm{\starV{\phi} - \EE_{p_{\phi}(.|s,a)}[\starV{\phi}(s')]}_{\infty}^2   \right].
\end{split}
\end{equation*}
We conclude the proof using Pinsker's inequality $\norm{p-q}_1 \leq \sqrt{2KL(p\| q)} $ along with the inequality $d_{H}(p,q)^2 \leq KL(p\| q)$ (see \cite{Reiss1989}).
\end{proof}

\subsection{Proof of Theorem \ref{theorem:pre-upper-bound}}
\begin{proof}
Consider the simplified problem (\ref{eq:analytic_problem}). Note that the constraint (\ref{eq:rewritten_conditions_with_kernels}) doesn't involve the pairs $(\Tilde{s},\Tilde{a}) \in\scal\times\acal\setminus\{(s,a), (s',\pi^\star(s'))_{s'\in \scal}\}$. One can easily show that any solution of the $\underset{u \in \mathcal{U}_{s a}}{\inf}$ part of (\ref{eq:analytic_problem}) must satisfy $KL_{\phi|\psi}(\Tilde{s},\Tilde{a}) = 0$ for these unconstrained pairs $(\Tilde{s},\Tilde{a}) \in\scal\times\acal\setminus\{(s,a), (\Tilde{s},\pi^\star(\Tilde{s}))_{\Tilde{s}\in \scal}\}$ (a trivial way to do it is by setting $\bigg(p_{\psi}(.|\Tilde{s},\Tilde{a}), q_{\psi}(.|\Tilde{s},\Tilde{a})\bigg) = \bigg(p_{\phi}(.|\Tilde{s},\Tilde{a}), q_{\phi}(.|\Tilde{s},\Tilde{a})\bigg)$). Therefore:
\begin{equation}
T^*(\phi)^{-1} = \underset{\omega \in \Sigma}{\sup}\ \underset{s, a\neq \pi^\star(s)}{\min}\ \underset{u \in \mathcal{U}_{s a}}{\inf} \omega_{s a}\textrm{KL}_{\phi|\psi}(s,a) + \underset{s'}{\sum}\ \omega_{s',\pi^\star_{\phi}(s')}\textrm{KL}_{\phi|\psi}(s',\pi^\star_{\phi}(s')). 
\label{eq:reduced_objective}
\end{equation}
We fix $s,a\neq \pi^\star(s)$ and derive a lower bound of $\underset{u \in \mathcal{U}_{s a}}{\inf} \omega_{s a}\textrm{KL}_{\phi|\psi}(s,a) + \underset{s'}{\sum}\ \omega_{s',\pi^\star_{\phi}(s')}\textrm{KL}_{\phi|\psi}(s',\pi^\star_{\phi}(s'))$. To do so, we rewrite the condition (\ref{eq:rewritten_conditions_with_kernels}) by expanding the expression of $dV^{\pi^\star}$as follows: 
\begin{equation*}
\begin{split}
& dr(s,a)  + \gamma dp(s,a)^{\top} \starV{\phi}
+ [\gamma p_{\psi}(s,a) - \mathbbm{1}(s)]^{\top} \left(I - \gamma P_{\psi}^{\pi^\star} \right)^{-1} \left[r_{\psi}^{\pi^\star} - r_{\phi}^{\pi^\star}\right] \\
& + [\gamma p_{\psi}(s,a) - \mathbbm{1}(s)]^{\top}\left[\left(I - \gamma P_{\psi}^{\pi^\star} \right)^{-1} - \left(I - \gamma P_{\phi}^{\pi^\star} \right)^{-1} \right]r_{\phi}^{\pi^\star} >
\Delta_{sa}.
\end{split}
\end{equation*}
We then write each of the four terms on the left-hand side as a "fraction" of $\Delta_{sa}$:
\begin{equation*}
\begin{cases}
dr(s,a) = \alpha_1 \Delta_{sa} \\
dp(s,a)^{\top} \starV{\phi} = \alpha_2 \Delta_{sa}\\
[\gamma p_{\psi}(s,a) - \mathbbm{1}(s)]^{\top} \left(I - \gamma P_{\psi}^{\pi^\star} \right)^{-1} \left[r_{\psi}^{\pi^\star} - r_{\phi}^{\pi^\star}\right] = \alpha_3 \Delta_{sa} \\
[\gamma p_{\psi}(s,a) - \mathbbm{1}(s)]^{\top}\left[\left(I - \gamma P_{\psi}^{\pi^\star} \right)^{-1} - \left(I - \gamma P_{\phi}^{\pi^\star} \right)^{-1} \right]r_{\phi}^{\pi^\star} = \alpha_4 \Delta_{sa} \\
\alpha_1 + \alpha_2 + \alpha_3 + \alpha_4 > 1
\end{cases}
\end{equation*}
We use Pinsker's inequality and Lemma \ref{lemma:hellinger} to lower bound each term.

\textbf{\underline{$1^{\textrm{st}}$ term.}} 
By Pinsker's inequality:
\begin{equation*}
\begin{split}
|dr(s,a)| = \left| \int_{0}^{1} u [q_{\psi}(u|s,a) -q_{\phi}(u|s,a)] \lambda(du) \ \right| &\leq \int_{0}^{1} |q_{\psi}(u|s,a) -q_{\phi}(u|s,a)| \ \lambda(du)\\
&\leq \sqrt{2 KL(q_{\phi}(.|s,a)\|q_{\psi}(.|s,a))}.
\end{split}
\end{equation*}
Thus:
\begin{equation}
\boxed{
\frac{1}{2}(\alpha_1 \Delta_{sa})^2 \leq KL(q_{\phi}(.|s,a)\|q_{\psi}(.|s,a))
}
\label{ineq:1}
\end{equation}

\textbf{\underline{$2^{\textrm{nd}}$ term.}} By Lemma \ref{lemma:hellinger}, we have: 
\begin{equation*}
 (\alpha_2 \Delta_{sa})^2 \leq 8 KL(p_{\phi}(s,a) \|  p_{\psi}(s,a)) \textrm{Var}_{s' \sim p_{\phi}(.|s,a)}[\starV{\phi}(s')] +4\sqrt{2} KL(p_{\phi}(s,a)\|  p_{\psi}(s,a))^{3/2}\textrm{MD}_{p_{\phi}(s,a)}[\starV{\phi}]^2.
\end{equation*}

Thus either: $$ \frac{1}{2}(\alpha_2 \Delta_{sa})^2 \leq 8 KL(p_{\phi}(s,a)\| p_{\psi}(s,a)) \textrm{Var}_{s' \sim p_{\phi}(.|s,a)}[\starV{\phi}(s')] $$ or $$ \frac{1}{2}(\alpha_2 \Delta_{sa})^2 \leq  4\sqrt{2} KL(p_{\phi}(s,a)\| p_{\psi}(s,a))^{3/2}\textrm{MD}_{p_{\phi}(s,a)}[\starV{\phi}]^2. $$
Therefore, we obtain: 
\begin{equation}
   \boxed{\min\left(\frac{\alpha_2^2  \Delta_{sa}^2}{16\textrm{Var}_{ p_{\phi}(s,a)}[\starV{\phi}] }, \frac{\alpha_2^{4/3} \Delta_{sa}^{4/3}}{2^{7/3}\textrm{MD}_{p_{\phi}(s,a)}[\starV{\phi}]^{4/3} }  \right) \leq KL\left(p_{\phi}(s,a)\| p_{\psi}(s,a)\right) 
   }
\label{ineq:2}
\end{equation}

\textbf{\underline{$3^{\textrm{rd}}$ term.}} We have: 
\begin{equation*}
\begin{split}
 |\alpha_3 | \Delta_{sa} &= \norm{[\gamma p_{\psi}(s,a) - \mathbbm{1}(s)]^{\top} \left(I - \gamma P_{\psi}^{\pi^\star} \right)^{-1} \left[r_{\psi}^{\pi^\star} - r_{\phi}^{\pi^\star}\right]}\\
 &\leq \norm{\gamma p_{\psi}(s,a) - \mathbbm{1}(s)}_{\infty}\times \norm{\left(I - \gamma P_{\psi}^{\pi^\star} \right)^{-1}}_{\infty} \times \norm{r_{\psi}^{\pi^\star} - r_{\phi}^{\pi^\star}}_{\infty} \\
 &\leq \frac{1}{1-\gamma}\norm{r_{\psi}^{\pi^\star} - r_{\phi}^{\pi^\star}}_{\infty},
\end{split}
\end{equation*}
which, following the same reasoning as the first term, implies: 
\begin{equation}
         \boxed{  \frac{(\alpha_3 \Delta_{sa} (1-\gamma) )^2}{2} \leq \ \underset{s \in \scal}{\max}\ KL\left(q_{\phi}(.|s,\pi^\star(s))\| q_{\psi}(.|s,\pi^\star(s))\right)
         }
    \label{ineq:3}
\end{equation}

\textbf{\underline{$4^{\textrm{th}}$ term (first bound).}} We have: 
\begin{equation}
|\alpha_4| \Delta_{sa} = \norm{[\gamma p_{\psi}(s,a) - \mathbbm{1}(s)]^{\top} \left[\left(I - \gamma P_{\psi}^{\pi^\star} \right)^{-1} - \left(I - \gamma P_{\phi}^{\pi^\star} \right)^{-1} \right]r_{\phi}^{\pi^\star}} \leq \norm{B}_{\infty},
\label{ineq:fourth_term_B}
\end{equation} 
where $B  =  \left[\left(I - \gamma P_{\psi}^{\pi^\star} \right)^{-1} - \left(I - \gamma P_{\phi}^{\pi^\star} \right)^{-1} \right]r_{\phi}^{\pi^\star}$. Hence: 
\begin{equation*}
\begin{split}
|\alpha_4| \Delta_{sa} \leq \norm{B}_{\infty} & = \gamma \norm{\left(I - \gamma P_{\psi}^{\pi^\star} \right)^{-1} \left[  P_{\psi}^{\pi^\star} - P_{\phi}^{\pi^\star}\right] \starV{\phi}}_{\infty} \\
 & \leq \frac{\underset{s' \in \scal}{\max} \ |dp(s',\pi^\star(s'))^{\top}\starV{\phi}|}{1-\gamma}.
\end{split}
\end{equation*}
Therefore, applying Lemma \ref{lemma:hellinger}, we get:
\begin{equation}
\boxed{ 
\begin{split}
         \min\left(\frac{\left[\alpha_4  \Delta_{sa}(1-\gamma)\right]^2}{16\textrm{Var}_{ max}^{\star}[\starV{\phi}]} , \ \frac{\alpha_4^{4/3} \Delta_{sa}^{4/3} (1-\gamma)^{4/3}}{2^{7/3}\textrm{MD}_{max}^{\star}[\starV{\phi}]^{4/3}}  \right) \\ 
    \leq \underset{s' \in \scal}{\max}\ KL\left(p_{\phi}(s',\pi^\star_{\phi}(s')) \|p_{\psi}(s',\pi^\star_{\phi}(s'))\right) 
\end{split}
}
\label{ineq:5}
\end{equation}

\textbf{\underline{$4^{\textrm{th}}$ term (second bound):}} We will now derive a second bound for the 4th term. Using Lemma \ref{lemma:second_bound_fourth_term}, we get: 
$$ |\alpha_4| \Delta_{sa} \leq \norm{B}_{\infty} \leq \frac{2^{5/2}\log(2)\textrm{KL}^{1/2}}{(1-\gamma)^{3/2}} + \frac{2^{3}\log(2)\gamma\textrm{KL}}{(1-\gamma)^{5/2}} + \frac{2^{5/4}\textrm{KL}^{3/4}\textrm{MD}_{max}^{\star}[\starV{\phi}]}{1-\gamma}$$
where $\textrm{KL} = \underset{s \in \scal}{\max} \ KL(p_{\phi}\left(s,\pi^\star_{\phi}(s))\|p_{\psi}(s,\pi^\star_{\phi}(s))\right)$. This means one of the three terms on the right-hand side is greater than $\frac{|\alpha_4| \Delta_{sa}}{3}$, which implies:
\begin{equation}
\boxed{ 
\begin{split}
       \min\left( \frac{\alpha_4^2 \Delta_{sa}^2(1-\gamma)^3}{288\log(2)^2},
        \frac{|\alpha_4| \Delta_{sa}(1-\gamma)^{5/2}}{24\log(2)},
        \frac{\alpha_4^{4/3} \Delta_{sa}^{4/3}(1-\gamma)^{4/3}}{2^{5/3} \times 3^{4/3}
        \textrm{MD}_{max}^{\star}[\starV{\phi}]^{4/3}}  \right) \\ 
    \leq \ \underset{s \in \scal}{\max} \ KL\left(p_{\phi}(s,\pi^\star_{\phi}(s)) \|p_{\psi}(s,\pi^\star_{\phi}(s))\right) 
\end{split}
}
\label{ineq:4}
\end{equation}

\textbf{\underline{Putting the individual lower bounds together:}}
Summing up all inequalities from (\ref{ineq:1}), (\ref{ineq:2}), (\ref{ineq:3}), (\ref{ineq:4}) and (\ref{ineq:5}), we deduce:
$$ \underset{\sum \alpha_i > 1 }{\inf} \ \sum_{i = 1}^{3} B_i + \max(B_4,B_5) \ \leq  \underset{u \in \mathcal{U}_{s a}}{\inf} \omega_{s a}\textrm{KL}_{\phi|\psi}(s,a)+\sum_{s'}\omega_{s',\pi^\star_{\phi}(s')}\textrm{KL}_{\phi|\psi}(s',\pi^\star_{\phi}(s'))$$ where
\begin{equation*}
\\
\begin{cases}
B_1 = \frac{1}{2}\omega_{s a}(\alpha_1 \Delta_{sa})^2 \\
B_2 = \omega_{s a} \min\left(\frac{\alpha_2^2  \Delta_{sa}^2}{16\textrm{Var}_{ p_{\phi}(s,a)}[\starV{\phi}] }, \frac{\alpha_2^{4/3} \Delta_{sa}^{4/3}}{2^{7/3}\textrm{MD}_{p_{\phi}(s,a)}[\starV{\phi}]^{4/3} }  \right) \\
\\
B_3 = \frac{1}{2}\underset{s}{\min}\ \omega_{s,\pi^\star(s)}\ (\alpha_3 \Delta_{sa} (1-\gamma) )^2 \\
\\
B_4 = \underset{s}{\min}\ \omega_{s,\pi^\star(s)}\ \min\left( \frac{\alpha_4^2 \Delta_{sa}^2(1-\gamma)^3}{288\log(2)^2},
        \frac{|\alpha_4| \Delta_{sa}(1-\gamma)^{5/2}}{24\log(2)},
        \frac{\alpha_4^{4/3} \Delta_{sa}^{4/3}(1-\gamma)^{4/3}}{2^{5/3} \times 3^{4/3}
        \textrm{MD}_{max}^{\star}[\starV{\phi}]^{4/3}}  \right) \\
\\ 
B_5 =  \underset{s}{\min}\ \omega_{s,\pi^\star(s)}\ \min\left(\frac{\left[\alpha_4  \Delta_{sa}(1-\gamma)\right]^2}{16\textrm{Var}_{ max}^{\star}[\starV{\phi}]} , \ \frac{\alpha_4^{4/3} \Delta_{sa}^{4/3} (1-\gamma)^{4/3}}{2^{7/3}\textrm{MD}_{max}^{\star}[\starV{\phi}]^{4/3}} \right)
\end{cases}
\end{equation*}
Notice that if $\alpha$ verifies the inequalities above, and $\sum_{i=1}^{4} \alpha_i > 1$, then the vector whose entries are $\displaystyle{\bigg(\frac{|\alpha_i|}{\sum_{j=1}^{4} |\alpha_j|}\bigg)}_{1\leq i\leq 4}$ also verifies these inequalities. Therefore we can restrict our attention to vectors $\alpha$ in the simplex $\Sigma_4$. In particular, we have $\alpha_i^2 \leq \alpha_i^{4/3} \leq \alpha_i$. Furthermore, we lower bound $\Delta_{sa}$ by $\Delta_{\min}$ in the terms $(B_j)_{3\leq j\leq 5}$. This simplifies the bound to:
\begin{equation}
\begin{split}
\underset{\omega \in \Sigma}{\sup}\ \underset{s, a\neq \pi^\star(s)}{\min}\ \underset{\alpha \in \Sigma_4 }{\inf} \ \sum_{i = 1}^{3} B'_i \alpha_i^2 + \max(B'_4,B'_5) \alpha_4^2 \ &\leq  \underset{\omega \in \Sigma}{\sup}\ \underset{s, a\neq \pi^\star(s)}{\min}\ \underset{u \in \mathcal{U}_{s a}}{\inf} \bigg(\omega_{s a}\textrm{KL}_{\phi|\psi}(s,a)\\
& \quad \quad \quad \quad \quad \quad \quad \quad +\sum_{s'}\omega_{s',\pi^\star_{\phi}(s')}\textrm{KL}_{\phi|\psi}(s',\pi^\star_{\phi}(s')) \bigg)\\
&= T^*(\phi)^{-1}
\end{split}
\end{equation}
where
\begin{equation*}
\\
\begin{cases}
B'_1 = \frac{1}{2}\omega_{s a}( \Delta_{sa}^2 \\
B'_2 = \omega_{s a} \min\left(\frac{ \Delta_{sa}^2}{16\textrm{Var}_{ p_{\phi}(s,a)}[\starV{\phi}] }, \frac{ \Delta_{sa}^{4/3}}{2^{7/3}\textrm{MD}_{p_{\phi}(s,a)}[\starV{\phi}]^{4/3}}  \right)\\
\\
B'_3 = \frac{1}{2}\underset{s}{\min}\ \omega_{s,\pi^\star(s)}\ (\Delta_{\min} (1-\gamma) )^2 \\

\\ 
B'_4 =  \underset{s}{\min}\ \omega_{s,\pi^\star(s)}\ \min\left( \frac{ \Delta_{\min}^2(1-\gamma)^3}{288\log(2)^2},
        \frac{ \Delta_{\min}(1-\gamma)^{5/2}}{24\log(2)},
        \frac{ \Delta_{\min}^{4/3}(1-\gamma)^{4/3}}{2^{5/3} \times 3^{4/3}
        \textrm{MD}_{max}^{\star}[\starV{\phi}]^{4/3}}  \right)\\
\\
B'_5 = \underset{s}{\min}\ \omega_{s,\pi^\star(s)}\ \min\left(\frac{ \Delta_{\min}^2 (1-\gamma)^2}{16\textrm{Var}_{ max}^{\star}[\starV{\phi}]} , \ \frac{\Delta_{\min}^{4/3} (1-\gamma)^{4/3}}{2^{7/3}\textrm{MD}_{max}^{\star}[\starV{\phi}]^{4/3}} \right)
\end{cases}
\end{equation*}
Solving the left-hand side problem above in $\alpha$, we get:
$$\underset{\omega \in \Sigma}{\sup}\ \underset{s, a\neq \pi^\star(s)}{\min}\ \bigg(\sum_{i = 1}^{3} \frac{1}{B'_i} + \min(\frac{1}{B'_4}, \frac{1}{B'_5})\bigg)^{-1} \leq T^*(\phi)^{-1}.$$
Therefore: 
$$ T^*(\phi) \leq \underset{\omega \in \Sigma}{\inf}\ \underset{s, a\neq \pi^\star(s)}{\max}\ \frac{T_1(s,a;\phi)+T_2(s,a;\phi)}{\omega_{s a}} + \frac{T_3(\phi) + T_4(\phi)}{\underset{s}{\min}\ \omega_{s,\pi^\star(s)}},$$
where
\begin{equation*}
\\
\begin{cases}
T_1(s,a;\phi) = \frac{2}{ \Delta_{sa}^2} \\
\\
T_2(s,a;\phi) =  \max\left(\frac{16\textrm{Var}_{ p_{\phi}(s,a)}[\starV{\phi}] }{\Delta_{sa}^2}, \frac{6\textrm{MD}_{p_{\phi}(s,a)}[\starV{\phi}]^{4/3}}{\Delta_{sa}^{4/3}}  \right)\\
\\
T_3(\phi) = \displaystyle{\frac{2}{\Delta_{\min}^2 (1-\gamma)^2}} \\

\\ 
T_4(\phi) =  \min\Bigg( V_1(\phi) , V_2(\phi)\Bigg), 
\end{cases}
\label{eq:four_terms}
\end{equation*}
and
\begin{align*}
& V_1(\phi)  =  \max\bigg(\frac{27}{\Delta_{\min}^2(1-\gamma)^3},
\frac{8}{\Delta_{\min}(1-\gamma)^{5/2}}, \frac{14 \mathrm{MD}_{\max}^{\star}[\starV{\phi}]^{4/3}}{\Delta_{\min}^{4/3} (1-\gamma)^{4/3}} \bigg),\\
& V_2(\phi) = \max\bigg(\frac{16\mathrm{Var}_{\max}^{\star}[\starV{\phi}]}{ \Delta_{\min}^2 (1-\gamma)^2}, \frac{6\mathrm{MD}_{\max}^{\star}[\starV{\phi}]^{4/3}}{\Delta_{\min}^{4/3} (1-\gamma)^{4/3}} \bigg).
\end{align*}
By Lemma \ref{lemma:delta_min}, we always have $\Delta_{\min} \leq 1$. In addition $\mathrm{MD}_{\max}^{\star}[\starV{\phi}] \leq \frac{1}{1-\gamma}$, hence $V_1(\phi) = \frac{27}{\Delta_{\min}^2(1-\gamma)^3}$, which simplifies the expression of $T_4(\phi)$:
$$
T_4(\phi) = \min\Bigg(\frac{27}{\Delta_{\min}^2(1-\gamma)^3}, \max\bigg(\frac{16\mathrm{Var}_{\max}^{\star}[\starV{\phi}]}{ \Delta_{\min}^2 (1-\gamma)^2}, \frac{6\mathrm{MD}_{\max}^{\star}[\starV{\phi}]^{4/3}}{\Delta_{\min}^{4/3} (1-\gamma)^{4/3}} \bigg)\Bigg).
$$
\end{proof}

\subsection{Second technical lemma: Contributions of transitions at optimal pairs to the sample complexity}
\begin{lemma}
Define:
$$ B  =  \left[\left(I - \gamma P_{\psi}^{\pi^\star} \right)^{-1} - \left(I - \gamma P_{\phi}^{\pi^\star} \right)^{-1} \right]r_{\phi}^{\pi^\star} \quad \textrm{and} \quad \textrm{KL} = \underset{s \in \scal}{\max} \ KL\left(p_{\phi}(s,\pi^\star(s))\| p_{\psi}(s,\pi^\star(s))\right).$$ 
Then we have:
\begin{equation}
 \norm{B}_{\infty} \leq \frac{2^{5/2}\log(2)\textrm{KL}^{1/2}}{(1-\gamma)^{3/2}} + \frac{2^{3}\log(2)\gamma\textrm{KL}}{(1-\gamma)^{5/2}} + \frac{2^{5/4}\textrm{KL}^{3/4}\textrm{MD}_{max}^{\star}[\starV{\phi}]}{1-\gamma}.
\end{equation}

\label{lemma:second_bound_fourth_term}
\end{lemma}
\begin{proof}
Let us further develop the expression of $B$:
    \begin{equation}
        \begin{split}
       B & =  \left[\left(I - \gamma P_{\psi}^{\pi^\star} \right)^{-1} - \left(I - \gamma P_{\phi}^{\pi^\star} \right)^{-1} \right]r_{\phi}^{\pi^\star}\\
       & = \left(I - \gamma P_{\psi}^{\pi^\star} \right)^{-1} \left[ \gamma P_{\psi}^{\pi^\star}-\gamma P_{\phi}^{\pi^\star}\right]\left(I - \gamma P_{\phi}^{\pi^\star} \right)^{-1}r_{\phi}^{\pi^\star} \\
       & = \gamma \left(I - \gamma P_{\psi}^{\pi^\star} \right)^{-1} \left[  P_{\psi}^{\pi^\star} - P_{\phi}^{\pi^\star}\right] \starV{\phi} \\
       & = \gamma \left[ \left(I - \gamma P_{\psi}^{\pi^\star} \right)^{-1} \left(I - \gamma P_{\phi}^{\pi^\star} \right) \right] \left(I - \gamma P_{\phi}^{\pi^\star} \right)^{-1} \left[  P_{\psi}^{\pi^\star} - P_{\phi}^{\pi^\star}\right] \starV{\phi} \\
       & \triangleq \gamma \ M_{\psi,\phi} \ \left(I - \gamma P_{\phi}^{\pi^\star} \right)^{-1} \left[P_{\psi}^{\pi^\star} - P_{\phi}^{\pi^\star}\right] \starV{\phi}.
        \end{split}
    \label{eq:B}
    \end{equation}
Notice that the quantity $\gamma \left(I - \gamma P_{\phi}^{\pi^\star} \right)^{-1} \left[P_{\psi}^{\pi^\star} - P_{\phi}^{\pi^\star}\right] \starV{\phi}$ is similar to the one that appears in Lemma 3 of \cite{azar2013minimax}, with $\psi$ playing the role of $\widehat{\phi}$ in this case. We will try to relate it to the variances of the value function in the $\phi$.
Define:

\begin{equation*}
\begin{cases}
M_{\psi,\phi} = \left(I - \gamma P_{\psi}^{\pi^\star} \right)^{-1} \left(I - \gamma P_{\phi}^{\pi^\star} \right), \\ 
 \textrm{KL} = \underset{s \in \scal}{\max} \ KL\left(p_{\phi}(s,\pi^\star(s))\| p_{\psi}(s,\pi^\star(s))\right), \\ 
v^{\pi}(s) = \gamma^2 Var_{s' \sim p_{\phi}(.|s, \pi(s))}[V_{\phi}^{\pi}(s')], \\
\sigma^{\pi}(s) = \gamma^2 Var_{(s',a') \sim p_{\phi}(.|s, \pi(s)) \otimes \pi(.|s') }[Q_{\phi}^{\pi}(s', a')].
\end{cases}
\end{equation*}

Using Lemma \ref{lemma:hellinger} and $\sqrt{a+b} \leq \sqrt{a} + \sqrt{b}$, we can write:  $\forall s \in \scal$, 
\begin{equation}
\begin{split}
\left|\gamma \left([P_{\psi}^{\pi^\star} - P_{\phi}^{\pi^\star}] \starV{\phi} \right)(s)\right| & = \left| \gamma dp(s,\pi^\star(s))^{\top} \starV{\phi} \right| \\
& \leq \gamma \sqrt{8 KL(p_{\phi}(s,\pi^\star(s))\| p_{\psi}(s,\pi^\star(s))) \ \textrm{Var}_{s' \sim p_{\phi}(.|s,\pi^\star(s))}[\starV{\phi}(s')]} \\
& \ \ +\gamma \sqrt{4\sqrt{2}KL(\ p_{\phi}(s,\pi^\star(s))\ || \ p_{\psi}(s,\pi^\star(s))\ )^{3/2}\textrm{MD}_{p_{\phi}(s,\pi^\star(s))}[\starV{\phi}]^2} \\
& \leq 2^{3/2}\textrm{KL}^{1/2} \sqrt{v^{\pi^\star}(s)} +  2^{5/4}\textrm{KL}^{3/4}\textrm{MD}_{max}^{\star}[\starV{\phi}] \\
& \leq  2^{3/2}\textrm{KL}^{1/2} \sqrt{\sigma^{\pi^\star}(s)} +  2^{5/4}\textrm{KL}^{3/4} \textrm{MD}_{max}^{\star}[\starV{\phi}],
\end{split}
\label{eq:Kl_sigma}
\end{equation}
where the last inequality comes from Total Variance theorem: 
\begin{equation*}
\begin{split}
\sigma^{\pi}(s) &= \gamma^2 Var_{(s',a') \sim p_{\phi}(.|s, \pi(s)) \otimes \pi(.|s') }[Q_{\phi}^{\pi}(s', a')] \\
&= \gamma^2  Var_{s'\sim p_{\phi}(.|s, \pi(s))}\bigg[\EE_{a' \sim \pi(.|s')}[Q_{\phi}^{\pi}(s', a')]\bigg] + \gamma^2 \EE_{s' \sim p_{\phi}(.|s, \pi(s))}\bigg[Var_{a' \sim \pi(.|s')}[Q_{\phi}^{\pi}(s', a')]\bigg] \\
&= v^{\pi}(s) + \gamma^2 \EE_{s' \sim p_{\phi}(.|s, \pi(s))}\bigg[Var_{a' \sim \pi(.|s')}[Q_{\phi}^{\pi}(s', a')]\bigg] \\
& \geq v^{\pi}(s).
\end{split}    
\end{equation*} 

Denote $\sqrt{\sigma^{\pi^\star}} \triangleq \left(\sqrt{\sigma^{\pi^\star}(s)}\right)_{s \in \scal}$. Then from (\ref{eq:B}) and (\ref{eq:Kl_sigma}), we deduce:
\begin{equation}
\begin{split}
\norm{B}_{\infty} &= \norm{M_{\psi,\phi} \left(I - \gamma P_{\phi}^{\pi^\star}\right)^{-1} \gamma[P_{\psi}^{\pi^\star} - P_{\phi}^{\pi^\star}] \starV{\phi} }_{\infty} \\
& \leq \norm{M_{\psi,\phi} \left(I - \gamma P_{\phi}^{\pi^\star}\right)^{-1} \bigg[2^{3/2}\textrm{KL}^{1/2}\sqrt{\sigma^{\pi^\star}} + 2^{5/4}\textrm{KL}^{3/4}\textrm{MD}_{max}^{\star}[\starV{\phi}] \mathbbm{1} \bigg]}_{\infty} \\
& \leq 2^{3/2}\textrm{KL}^{1/2} \norm{M_{\psi,\phi} }_{\infty} \norm{\left(I - \gamma P_{\phi}^{\pi^\star} \right)^{-1}\sqrt{\sigma^{\pi^\star}} }_{\infty} + 2^{5/4}\textrm{KL}^{3/4}\textrm{MD}_{max}^{\star}[\starV{\phi}] \norm{M_{\psi,\phi} \ \left(I - \gamma P_{\phi}^{\pi^\star} \right)^{-1} \mathbbm{1}}_{\infty} \\
& = 2^{3/2}\textrm{KL}^{1/2} \norm{M_{\psi,\phi} }_{\infty} \norm{\left(I - \gamma P_{\phi}^{\pi^\star} \right)^{-1}\sqrt{\sigma^{\pi^\star}} }_{\infty} + 2^{5/4}\textrm{KL}^{3/4}\textrm{MD}_{max}^{\star}[\starV{\phi}] \norm{\left(I - \gamma P_{\psi}^{\pi^\star} \right)^{-1}\mathbbm{1}}_{\infty} \\
& \leq 2^{3/2}\textrm{KL}^{1/2} \norm{M_{\psi,\phi} }_{\infty} \norm{\left(I - \gamma P_{\phi}^{\pi^\star} \right)^{-1}\sqrt{\sigma^{\pi^\star}} }_{\infty} + \frac{2^{5/4}}{1-\gamma}\textrm{KL}^{3/4}\textrm{MD}_{max}^{\star}[\starV{\phi}].
\end{split}
\label{eq:intermediate_1}
\end{equation}
Now observe that:
\begin{equation}
\begin{split}
\norm{M_{\psi,\phi} }_{\infty} &= \norm{\left(I - \gamma P_{\psi}^{\pi^\star} \right)^{-1} \left(I - \gamma P_{\phi}^{\pi^\star} \right) }_{\infty} \\
& = \norm{I -\gamma\left(I - \gamma P_{\psi}^{\pi^\star} \right)^{-1} \left(P_{\phi}^{\pi^\star} - P_{\psi}^{\pi^\star} \right) }_{\infty} \\
& \leq 1 + \frac{\gamma \norm{P_{\phi}^{\pi^\star} - P_{\psi}^{\pi^\star}}_{\infty}}{1-\gamma} \\
& \leq 1 + \frac{\gamma(2\textrm{KL})^{1/2}}{1-\gamma},
\end{split}
\label{eq:intermediate_3}
\end{equation}
where the last inequality stems from Pinsker's inequality. 
Next we recall a variance inequality from \cite{azar2013minimax}:
\begin{lemma}(Lemma 8, \cite{azar2013minimax})
\begin{equation*}
\norm{\left(I - \gamma P_{\phi}^{\pi^\star} \right)^{-1}\sqrt{\sigma^{\pi^\star}} }_{\infty} \leq \frac{2\log(2)}{(1-\gamma)^{3/2}}.   
\end{equation*}
\label{lemma:munos_variance}
\end{lemma}
Summing up equations (\ref{eq:intermediate_1}), (\ref{eq:intermediate_3}) and Lemma \ref{lemma:munos_variance}, we get:
\begin{equation}
\norm{B}_{\infty} \leq \frac{2^{5/2}\log(2)\textrm{KL}^{1/2}}{(1-\gamma)^{3/2}} + \frac{2^{3}\log(2)\gamma\textrm{KL}}{(1-\gamma)^{5/2}} + \frac{2^{5/4}\textrm{KL}^{3/4}\textrm{MD}_{max}^{\star}[\starV{\phi}]}{1-\gamma}.
\end{equation}
\end{proof}

\subsection{Third technical lemma: The minimum gap is smaller than 1}
\begin{lemma}
$\Delta_{\min} \leq 1$.
\label{lemma:delta_min}
\end{lemma}
\begin{proof}
By contradiction, suppose $\Delta_{\min} > 1$, then:
$$
\forall s,a\neq \pi^\star (s),\ \Delta_{sa} = \starV{\phi}(s) - \starQ{\phi}(s,a) > 1.
$$
This means that for all policies $\pi \in \{ \pi\: \forall s \in \scal,\ \pi(s) \neq \pi^\star(s) \}$, we have:
$$
\forall s \in \scal,\ \starQ{\phi}(s,\pi(s)) < \starV{\phi}(s) - 1.
$$
Using Bellman operator, the above inequality becomes:
$$ \bcal_{\phi}^{\pi}\starV{\phi} < \starV{\phi} - \mathbbm{1}. 
$$
By induction, using that the monotonicity of Bellman operator:
$$ \forall n \geq 1,\ \bigg(\bcal_{\phi}^{\pi}\bigg)^n \starV{\phi} < \starV{\phi} -(\sum_{i=0}^{n-1} \gamma^i) \mathbbm{1}.  
$$
Therefore:
\begin{equation*}
\begin{split}
\forall \pi \in \{ \pi\: \forall s \in \scal,\ \pi(s)\neq \pi^\star(s) \},\ V_{\phi}^{\pi} &= \underset{n\to\infty}{\lim}\
\bigg(\bcal_{\phi}^{\pi}\bigg)^{n+1} \starV{\phi} \\ 
&\leq \underset{n\to\infty}{\lim}\
\bigg(\bcal_{\phi}^{\pi}\bigg) \bigg[\starV{\phi} - (\sum_{i=0}^{n-1} \gamma^i)\mathbbm{1}\bigg]\\
& = \bigg(\bcal_{\phi}^{\pi}\bigg) \starV{\phi} - \underset{n\to\infty}{\lim}\
(\sum_{i=1}^{n} \gamma^i)\mathbbm{1}\\
&= \bigg(\bcal_{\phi}^{\pi}\bigg)\starV{\phi} - \frac{\gamma}{1-\gamma}\mathbbm{1}\\
&< \starV{\phi} - \frac{1}{1-\gamma}\\
&< 0.
\end{split}
\end{equation*}
We obtained a contradiction. Thus, $\Delta_{\min} \leq 1$.
\end{proof}

\subsection{Proof of  Corollary\ref{corollary:upper_bound}}

\begin{proof}
The $\omega$ solving the problem in the right-hand side of (\ref{eq:upper_bound_problem}) clearly verifies:
$$ \ \forall s \in \scal, \ \ \omega_{s,\pi^\star(s)} = \underset{s'}{\min}\ \omega_{s',\pi^\star(s')} \triangleq \omega_0. 
$$ 
The problem of Theorem \ref{theorem:pre-upper-bound} then rewrites as:
\begin{align}
&\inf\limits_{\omega_0}\quad \max\limits_{s, a\neq \pi^\star(s)}\ \frac{H_{sa}}{\omega_{sa}} + \frac{H^\star}{S\omega_0}\\
&(\omega_{\Tilde{s},\Tilde{a}})_{\Tilde{s},\Tilde{a}\neq \pi^\star(\Tilde{s})}
\label{new_proof_problem}
\end{align}
where $H_{s a} = T_1(s,a;\phi) + T_2(s,a;\phi)$ and $H^\star =  S(T_3(\phi) + T_4(\phi))$. We reformulate (\ref{new_proof_problem}) as a convex program: 
\begin{equation*}
\begin{aligned}
&\inf\limits_{t, \omega_0}  \quad \quad t + \frac{H^\star}{S \omega_0}\\
&(\omega_{s a})_{s,a\neq \pi^\star(s)}\\
&\hbox{s.t. } \omega^{\top}\mathbbm{1} = 1, \\
&t \geq \frac{H_{s a}}{\omega_{s a}}, \forall s,a\neq \pi^\star (s)
\end{aligned}
\end{equation*}

%
Using KKT conditions, one can easily derive the expression of the solution:
\begin{equation}
\begin{cases}
\overline{\omega}_{s,a} =  \frac{H_{s a}}{\underset{s,a\neq \pi^\star (s)}{\sum}H_{s a}\ +\ \sqrt{H^\star \left(\underset{s,a\neq \pi^\star (s)}{\sum}H_{s a}\right)} } \quad \forall s, a\neq \pi^\star(s), \\
\\
\overline{\omega}_{s,\pi^\star(s)} = \frac{1}{S} \times \frac{\sqrt{H^\star \left(\underset{s,a\neq \pi^\star (s)}{\sum}H_{s a}\right)} }
{\underset{s, a\neq \pi^\star(s)}{\sum}H_{s a}\ +\ \sqrt{H^\star \left(\underset{s, a\neq \pi^\star(s)}{\sum}H_{s a}\right)}} \quad \forall s \in \scal.
\end{cases}
\label{eq:optimal_weights_full_def}
\end{equation}
The value $V_P$ of the program is:
\begin{equation*}
   V_P = \underset{s, a\neq \pi^\star(s)}{\sum}H_{s a} + H^\star + 2\sqrt{H^\star \left(\underset{s, a\neq \pi^\star(s)}{\sum}H_{s a}\right)} \leq 2 \bigg(\underset{s, a\neq \pi^\star(s)}{\sum}H_{s a} + H^\star \bigg) \triangleq U(\phi).
\end{equation*}
\end{proof}

\section{PAC Guarantee:}
\label{AppendixB}

\normalsize 
\subsection{Proof of Theorem \ref{Theorem:stopping_rule}}
First we recall two concentration inequalities and a technical lemma that we will be using. The first two lemmas are taken from \cite{jonsson2020planning}. The third lemma is immediate. \\
Define the threshold function $x(n,\delta,m) = \log(1/\delta) + (m-1)\log\bigg(e(1+n/(m-1))\bigg) $
\begin{lemma} (Proposition 2, \cite{jonsson2020planning})
For all distributions $q$ of mean $r$ supported on the unit interval, for all $\delta \in [0,1]$:
$$ \PP\bigg(\exists n \in \mathbb{N}\ n \kl(\widehat{r}_n, r) > x(\delta,n,2) \bigg) \leq  \delta.$$

\label{lemma:concentration rewards}
\end{lemma}

\begin{lemma}(Proposition 1, \cite{jonsson2020planning})
Let $P$ be a distribution over a finite set $\scal$, and $(X_i)_{i\in \mathbb{N}}$ be i.i.d. variables with distribution $P$. For $s\in \scal$, denote by $\widehat{P}_n=(\widehat{p}_n(s))_{s\in \scal}$ the empirical estimate of $P$ from the first $n$ samples. Then for all $\delta \in [0,1]$ :
$$\PP\left(\exists n \in \mathbb{N}\ n KL(\widehat{P}_n\ ||\ P) > x(\delta,n,S)  \right) \leq \delta,$$
where we used $S$ as a shorthand for $|\scal|$.
\label{lemma:concentration transitions}
\end{lemma}

\begin{lemma}
Let $(\rho_i)_{1 \leq i \leq 4} \in \mathbbm{R}_{+}^4$. Then: 
\begin{equation*}
\forall \alpha \in \Sigma_4\ \exists i \in [|0,4|],\ \rho_i < \alpha_i^2 \iff \sum_{i=0}^{4} \sqrt{\rho_i} < 1.
\end{equation*}
\label{technical_lemma_3}
\end{lemma}

We are now ready to prove Theorem \ref{Theorem:stopping_rule} :

\begin{proof}
Recall the definition of the "correctness" event: 
$$ \mathcal{E}_t = \bigg( \forall \alpha \in \Sigma_4 \ \forall s,a\neq \widehat{\pi}_t^{\star}(s),\  \rho_1(\widehat{\phi}_t,\phi)(s,a) < \alpha_1^2 \textrm{ or } \rho_2(\widehat{\phi}_t,\phi)(s,a) < \alpha_2^2 \textrm{ or } \rho_3(\widehat{\phi}_t,\phi) < \alpha_3^2 \textrm{ or } \rho_4(\widehat{\phi}_t,\phi) < \alpha_4^2  \bigg) $$
where: 
\begin{equation*}
\begin{cases}
\rho_1(\phi,\psi)(s,a) \triangleq T_1(s,a;\phi) KL(r_{\phi}(s,a)||r_{\psi}(s,a)), \\
\rho_2(\phi,\psi)(s,a) \triangleq  T_2(s,a;\phi) KL(p_{\phi}(s,a)||p_{\psi}(s,a)),\\
\rho_3(\phi,\psi)(s) \triangleq  T_3(\phi) KL\bigg(r_{\phi}(s,\pi_{\phi}^\star(s))\ || \ r_{\psi}(s,\pi_{\phi}^\star(s))\bigg), \\
\rho_4(\phi,\psi)(s) \triangleq  T_4(\phi) KL\bigg(p_{\phi}(s,\pi_{\phi}^\star(s))\ || \ p_{\psi}(s,\pi_{\phi}^\star(s))\bigg), \\
\rho_3(\phi,\psi) \triangleq \underset{s \in \scal}{\max}\ \rho_3(\phi,\psi)(s), \\
\rho_4(\phi,\psi) \triangleq \underset{s \in \scal}{\max}\ \rho_4(\phi,\psi)(s).
\end{cases}
\end{equation*}


Applying Lemma \ref{technical_lemma_3}, we can simplify the event  $\mathcal{E}_t$:
\begin{align}
 \mathcal{E}_t 
 & = \underset{s,a\neq \widehat{\pi}_t^{\star}(s)}{\bigcap}\ \bigg( \sqrt{\rho_1(\widehat{\phi}_t,\phi)(s,a)} + \sqrt{\rho_2(\widehat{\phi}_t,\phi)(s,a)} + \sqrt{\rho_3(\widehat{\phi}_t,\phi)} + \sqrt{\rho_4(\widehat{\phi}_t,\phi)} < 1 \bigg) \\
 & = \underset{s,a\neq \widehat{\pi}_t^{\star}(s)}{\bigcap}\ \ \underset{s',s"\in \scal}{\bigcap}\ \bigg( \sqrt{\rho_1(\widehat{\phi}_t,\phi)(s,a)} + \sqrt{\rho_2(\widehat{\phi}_t,\phi)(s,a)} + \sqrt{\rho_3(\widehat{\phi}_t,\phi)(s')} + \sqrt{\rho_4(\widehat{\phi}_t,\phi)(s")} < 1 \bigg). 
\label{eq:simplified_correctness_event}
\end{align}

Define the stopping event:
\begin{equation}
\begin{split}
 \textrm{STOP}_t &= \Bigg\{ \underset{s,a\neq \widehat{\pi}_t^{\star}(s)}{\max} \frac{\sqrt{\widehat{T_1}(s,a) x(\delta',n_t(s,a),2)  }+ \sqrt{\widehat{T_2}(s,a) x(\delta',n_t(s,a),S)}}{\sqrt{n_t(s,a)}}\\
     &+\ \underset{s\in \scal}{\max} \frac{\sqrt{\widehat{T_3} x(\delta',n_t(s,\widehat{\pi}_{t}^{*}(s)),2)  }+ \sqrt{\widehat{T_4} x(\delta',n_t(s,\widehat{\pi}_{t}^{*}(s)),S)}}{\sqrt{n_t(s,\widehat{\pi}_{t}^{*}(s))}} < 1 \Bigg\} \\ 
      &= \Bigg\{ \underset{s,a\neq \widehat{\pi}_t^{\star}(s)}{\max} \frac{\sqrt{\widehat{T_1}(s,a) x(\delta',n_t(s,a),2)  }+ \sqrt{\widehat{T_2}(s,a) x(\delta',n_t(s,a),S)}}{\sqrt{n_t(s,a)}}\\
     &+ \underset{s\in \scal}{\max} \frac{\sqrt{\widehat{T_3} x(\delta',n_t(s,\widehat{\pi}_{t}^{*}(s)),2)}  }{\sqrt{n_t(s,\widehat{\pi}_{t}^{*}(s))}} +\ \underset{s\in \scal}{\max} \frac{\sqrt{\widehat{T_4} x(\delta',n_t(s,\widehat{\pi}_{t}^{*}(s)),S)}  }{\sqrt{n_t(s,\widehat{\pi}_{t}^{*}(s))}}  < 1 \Bigg\} \\
\end{split} 
\label{eq:simplified_STOP_event}
\end{equation}
where the last equality stems from the fact that both $n\to \frac{\sqrt{\widehat{T_3} x(\delta',n,2)}}{\sqrt{n}}$ and $n \to \frac{\sqrt{\widehat{T_4} x(\delta',n,S)}}{\sqrt{n}}$ are decreasing as soon as $n \geq 7(\scal-1)$, therefore reaching their maximum at the same point.
From the proof of Theorem \ref{theorem:pre-upper-bound} (refer to Equations (\ref{ineq:1})-(\ref{ineq:2})-(\ref{ineq:3})-(\ref{ineq:4})-(\ref{ineq:5})), we have the following "correctness' property:
\begin{align}
 \left(\phi \in \mathrm{Alt}(\widehat{\phi}_{t})\right)\ \subset \mathcal{E}_t^{c},
\label{eq:correctness_inclusion}
\end{align}
where $\mathcal{E}_t^{c}$ stands for the complement of event $\mathcal{E}$. Therefore: 
\begin{equation*}
\begin{split}
& (\tau_{\delta} < \infty) \cap (\widehat{\pi}_{\tau_{\delta}}^\star \neq \pi^\star) = \left(\exists t \geq 1,\ \mathrm{STOP}_t \textrm{ and } \phi \in \mathrm{Alt}(\widehat{\phi}_{t}) \right)\\ 
& \mbox{\LARGE$\subset$} \bigg(\exists t \geq 1,\ \mathrm{STOP}_t \cap \mathcal{E}_t^{c} \bigg)\\
&=  \Bigg( \exists t \geq 1,\ \underset{s,a\neq \widehat{\pi}_t^{\star}(s)}{\bigcup}\ \ \underset{s',s"\in \scal}{\bigcup}\ \Bigg(\bigg( \sqrt{\rho_1(\widehat{\phi}_t,\phi)(s,a)} + \sqrt{\rho_2(\widehat{\phi}_t,\phi)(s,a)} + \sqrt{\rho_3(\widehat{\phi}_t,\phi)(s')} + \sqrt{\rho_4(\widehat{\phi}_t,\phi)(s")} \geq 1 \bigg) \\
& \quad \quad \quad \quad \quad \quad \quad \quad \quad \quad \quad \quad \quad \quad \cap \textrm{STOP}_t \Bigg) \Bigg) \\
& \mbox{\LARGE$\subset$} \Bigg( \exists t \geq 1,\ \underset{s,a\neq \widehat{\pi}_t^{\star}(s)}{\bigcup}\ \ \underset{s',s"\in \scal}{\bigcup}\ \bigg(\mathcal{E}_{1,t}(s,a) \cup \mathcal{E}_{2,t}(s,a) \cup \mathcal{E}_{3,t}(s') \cup \mathcal{E}_{4,t}(s")  \bigg) \Bigg)\\
&\mbox{\LARGE$\subset$} \underset{(s,a)\in \scal\times\acal}{\bigcup}\ \underset{s',s"\in \scal}{\bigcup}\ \Bigg(\bigg(\exists t \geq 1,\ \mathcal{E}_{1,t}(s,a)\cap \big(a = \widehat{\pi}_t^{\star}(s)\big)\bigg)  \cup \bigg(\exists t \geq 1,\ \mathcal{E}_{2,t}(s,a)\cap \big(a = \widehat{\pi}_t^{\star}(s)\big)\bigg) \\
& \quad \quad \quad \quad \quad \quad \quad \quad \quad \quad \cup \bigg(\exists t \geq 1,\ \mathcal{E}_{3,t}(s') \bigg) \cup \bigg(\exists t \geq 1,\ \mathcal{E}_{4,t}(s") \bigg)  \Bigg),
\end{split}
\end{equation*}
where 
\begin{equation*}
\begin{cases}
\mathcal{E}_{1,t}(s,a) \triangleq \displaystyle{ \bigg\{ \sqrt{\rho_1(\widehat{\phi}_t,\phi)(s,a)} > \frac{\sqrt{\widehat{T_1}(s,a) x(\delta',n_t(s,a),2)}}{\sqrt{n_t(s,a)}} \bigg\} }, \quad \forall (s,a) \notin \ocal(\widehat{\phi}_t), \\ \\
\mathcal{E}_{2,t}(s,a) \triangleq \displaystyle{ \bigg\{ \sqrt{\rho_2(\widehat{\phi}_t,\phi)(s,a)} > \frac{\sqrt{\widehat{T_2}(s,a) x(\delta',n_t(s,a),S)}}{\sqrt{n_t(s,a)}} \bigg\} }, \quad \forall (s,a) \notin \ocal(\widehat{\phi}_t), \\ \\
\mathcal{E}_{3,t}(s) \triangleq \displaystyle{ \bigg\{ \sqrt{\rho_3(\widehat{\phi}_t,\phi)(s)} > \frac{\sqrt{\widehat{T}_{3}x(\delta',n_t(s,\widehat{\pi}_{t}^{*}(s)),2)}}{\sqrt{n_t(s,\widehat{\pi}_{t}^{*}(s))}}  \bigg\} }, \quad \forall s \in \scal, \\ \\
\mathcal{E}_{4,t}(s) \triangleq \displaystyle{ \bigg\{ \sqrt{\rho_4(\widehat{\phi}_t,\phi)(s)} > \frac{\sqrt{\widehat{T}_{4} x(\delta',n_t(s,\widehat{\pi}_{t}^{*}(s)),S)}}{\sqrt{n_t(s,\widehat{\pi}_{t}^{*}(s))}} \bigg\} }, \quad \forall s \in \scal.
\end{cases}
\end{equation*}

Therefore:
\begin{equation*}
\begin{split}
\PP_{\phi}(\tau_{\delta} < \infty,\widehat{\pi}_{\tau_{\delta}}^\star \neq \pi^\star_{\phi} ) &\leq \underset{(s,a)\in \scal\times\acal}{\sum}\ \underset{s',s" \in \scal}{\sum}\bigg[ \PP\bigg(\exists t \geq 1,\ \mathcal{E}_{1,t}(s,a)\cap \big(a =  \widehat{\pi}_t^{\star}(s)\big)\bigg)\\
&+\PP\bigg(\exists t \geq 1,\ \mathcal{E}_{2,t}(s,a)\cap \big(a =  \widehat{\pi}_t^{\star}(s)\big)\bigg)
 + \PP\bigg(\exists t \geq 1,\ \mathcal{E}_{3,t}(s') \bigg) + \PP\bigg(\exists t \geq 1,\ \mathcal{E}_{4,t}(s") \bigg)\bigg] \\
&\leq \underset{(s,a)\in \scal\times\acal}{\sum}\ \underset{s',s" \in \scal}{\sum} 4\delta'\\
& =4 S^3 A \delta' \triangleq \delta,
\end{split}
\end{equation*}

where in the second inequality we have used the concentration inequalities (\ref{eq:first-stopping-condition}), (\ref{eq:second-stopping-condition}), (\ref{eq:third-stopping-condition}) and (\ref{eq:fourth-stopping-condition}). We detail the derivation of this second inequality below:

{\bf \underline{First term}.} Using Lemma \ref{lemma:concentration rewards}, for $\delta' = \frac{\delta}{4S^3 A}$, we have:
\begin{equation}
\begin{split}
& \PP\left(\exists t \geq 1,\ \sqrt{\rho_1(\widehat{\phi}_t,\phi)(s,a)} > \frac{\sqrt{\widehat{T_1}(s,a) x(\delta',n_t(s,a),2)}}{\sqrt{n_t(s,a)}}  \right)\\
&= \PP\bigg(\exists t \geq 1,\ n_t(s,a) \kl(\widehat{r}_{n_t(s,a)}(s,a), r(s,a)) >  x\left(\delta',n_t(s,a),2\right) \bigg)\\
&\leq \PP\bigg( \exists n \in \mathbb{N},\ n \kl(\widehat{r}_{n}(s,a), r(s,a)) > x\left(\delta',n,2\right) \bigg) \\
& \leq \delta'.
\end{split}
\label{eq:first-stopping-condition}
\end{equation}

{\bf \underline{Second term}.} Using Lemma \ref{lemma:concentration transitions}, we get:
\begin{equation}
\begin{split}
&\PP\bigg(\exists t \geq 1,\ \sqrt{\rho_2(\widehat{\phi}_t,\phi)(s,a)} > \frac{\sqrt{\widehat{T_2}(s,a) x(\delta',n_t(s,a),S)}}{\sqrt{n_t(s,a)}}  \bigg)\\
&= \PP\bigg(\exists t \geq 1,\ n_t(s,a) KL\left(\widehat{p}_{n_t(s,a)}(s,a)\| p(s,a)\right) >  x\left(\delta',n_t(s,a),S \right) \bigg)\\
&\leq \PP\bigg( \exists n \in \mathbb{N},\ KL\left(\widehat{p}_{n}(s,a)\| p(s,a)\right) > x\left(\delta',n,S \right) \bigg) \\
& \leq \delta'.
\end{split}
\label{eq:second-stopping-condition}
\end{equation}

{\bf \underline{Third term}.} Following the same reasoning as in the first term we get: 
\begin{equation}
\forall s \in \scal,\ \PP\left(\exists t \geq 1,\ \sqrt{\rho_3(\widehat{\phi}_t,\phi)(s)} > \frac{\sqrt{\widehat{T}_{3,t}x(\delta',n_t(s,\widehat{\pi}_t(s)),2)}}{\sqrt{n_t(s,\widehat{\pi}^{*}(s))}}  \right) \leq \delta'.
\label{eq:third-stopping-condition}
\end{equation}

{\bf \underline{Fourth term}.} Following the same reasoning as in the second term we get: 
\begin{equation}
\forall s \in \scal,\ \PP\left(\exists t \geq 1,\ \sqrt{\rho_4(\widehat{\phi}_t,\phi)(s)} > \frac{\sqrt{\widehat{T}_{4,t}x(\delta',n_t(s,\widehat{\pi}_t(s)),S) \bigg)}}{\sqrt{n_t(s,\widehat{\pi}^{*}(s))}}  \right) \leq \delta'.
\label{eq:fourth-stopping-condition}
\end{equation}
\end{proof}

\section{Sample complexity of KLB-TS}
\label{AppendixC}
\normalsize
In the following, we use the notation: $y(n,m) \triangleq (m-1) + (m-1) \log(1+n/(m-1))$. Hence the threshold function can be rewritten as: $x(\delta,n,m) = \log(1/\delta) + y(n,m)$. 

We start this section by a technical lemma that is later used in the proof of Proposition \ref{proposition:almost_sure_sample_complexity} and Theorem \ref{Theorem:Upper bound in expectation}.

\begin{lemma}
For all $\phi$ in $\Phi$,
\begin{equation*}
 \bigg(\underset{s,a\neq \pi^\star (s)}{\max} \frac{\sqrt{T_1(s,a;\phi)}+ \sqrt{T_2(s,a;\phi)}}{\sqrt{\overline{\omega}_{s,a}}} +\ \underset{s\in \scal}{\max} \frac{\sqrt{T_3(\phi)}+ \sqrt{T_4(\phi) }}{\sqrt{\overline{\omega}_{s,\pi^\star(s)}}} \bigg)^2  \leq 4 U(\phi).
\end{equation*}
\label{lemma:technical_bound}
\end{lemma}
\begin{proof}
Denote by $\mathrm{LHS}$ the left-hand side term above.
Using $(A+B)^2 \leq 2(A^2 + B^2)$ twice, and $(\underset{x}{\max} f(x))^2 = \underset{x}{\max} f(x)^2$ for non-negative $f$, we write: 
\begin{equation*}
\begin{split}
 \mathrm{LHS} &\leq 2\Bigg(\bigg(\underset{s,a\neq \pi^\star (s)}{\max} \frac{\sqrt{T_1(s,a;\phi)}+ \sqrt{T_2(s,a;\phi)}}{\sqrt{\overline{\omega}_{s,a}}}\bigg)^2 +\ \bigg(\underset{s\in \scal}{\max} \frac{\sqrt{T_3(\phi)}+ \sqrt{T_4(\phi) }}{\sqrt{\overline{\omega}_{s,\pi^\star(s)}}}\bigg)^2 \Bigg)\\
 &= 2\Bigg(\underset{s,a\neq \pi^\star (s)}{\max} \bigg(\frac{\sqrt{T_1(s,a;\phi)}+ \sqrt{T_2(s,a;\phi)}}{\sqrt{\overline{\omega}_{s,a}}}\bigg)^2 +\ \underset{s\in \scal}{\max} \bigg(\frac{\sqrt{T_3(\phi)}+ \sqrt{T_4(\phi) }}{\sqrt{\overline{\omega}_{s,\pi^\star(s)}}}\bigg)^2 \Bigg)\\
 &\leq 4 \Bigg(\underset{s,a\neq \pi^\star (s)}{\max} \frac{T_1(s,a;\phi)+ T_2(s,a;\phi)}{\overline{\omega}_{s,a}} +\ \underset{s\in \scal}{\max} \frac{T_3(\phi) + T_4(\phi) }{\overline{\omega}_{s,\pi^\star(s)}} \Bigg)\\
 & \leq 4 U(\phi),
\end{split}
\end{equation*}
where the last inequality comes from Corollary \ref{corollary:upper_bound}.
\end{proof}

\subsection{Proof of Proposition \ref{proposition:almost_sure_sample_complexity}}
\begin{proof}
Recall the stopping condition: 
\begin{equation*}
\begin{split}
\tau_{\delta} =  \inf\Bigg\{& t \in \mathbb{N}\ :
     \underset{s,a\neq \widehat{\pi}_{t}^{\star}(s)}{\max} \frac{\sqrt{\widehat{T_1}(s,a) x(\delta',n_t(s,a),2)  }+ \sqrt{\widehat{T_2}(s,a) x(\delta',n_t(s,a),S)}}{\sqrt{n_t(s,a)}}\\
     &+\ \underset{s\in \scal}{\max} \frac{\sqrt{\widehat{T_3} x(\delta',n_t(s,\widehat{\pi}_t^\star(s)),2)  }+ \sqrt{\widehat{T_4} x\bigg(\delta',n_t(s,\widehat{\pi}_t^\star(s)),S\bigg)}}{\sqrt{n_t(s,\widehat{\pi}_t^\star(s))}} \leq 1 \Bigg\}.
\end{split}
\end{equation*}
First we derive a convenient upper-bound of the left-hand-side term of the inequality above (which we denote by $\mathrm{LHS}_t$).\\
Rewrite the definition of $x(\delta,n,m) = \log(1/\delta) + (m-1) + (m-1) \log(1+n/(m-1)) \triangleq \log(1/\delta) + y(n,m)$. Then, using the fact that $\sqrt{A+B} \leq \sqrt{A} + \sqrt{B}$, we have: 
\begin{equation}
\begin{split}
\mathrm{LHS}_t \leq &  \sqrt{\log(\delta')} \Bigg( \underset{s,a\neq \widehat{\pi}_{t}^{\star}(s)}{\max} \frac{\sqrt{\widehat{T_1}(s,a)}+ \sqrt{\widehat{T_2}(s,a) }}{\sqrt{n_t(s,a)}} +\ \underset{s\in \scal}{\max} \frac{\sqrt{\widehat{T_3}}+ \sqrt{\widehat{T_4} }}{\sqrt{n_t(s,\widehat{\pi}_t^\star(s))}} \Bigg) \\ \\
&+ \underset{s,a\neq \widehat{\pi}_{t}^{\star}(s)}{\max} \frac{\sqrt{\widehat{T_1}(s,a) y(n_t(s,a),2)}+ \sqrt{\widehat{T_2}(s,a) y(n_t(s,a),S)}}{\sqrt{n_t(s,a)}}\\
&+\ \underset{s\in \scal}{\max} \frac{\sqrt{\widehat{T_3}y(n_t(s,\widehat{\pi}_t^\star(s)),2)}+ \sqrt{\widehat{T_4} y(n_t(s,\widehat{\pi}_t^\star(s)),S) }}{\sqrt{n_t(s,\widehat{\pi}_t^\star(s))}} \\ \\
& \triangleq \sqrt{\log(\delta')} \Bigg( \underset{s,a\neq \widehat{\pi}_{t}^{\star}(s)}{\max} \frac{\sqrt{\widehat{T_1}(s,a)}+ \sqrt{\widehat{T_2}(s,a) }}{\sqrt{n_t(s,a)}}\ +\ \underset{s\in \scal}{\max} \frac{\sqrt{\widehat{T_3}}+ \sqrt{\widehat{T_4} }}{\sqrt{n_t(s,\widehat{\pi}_t^\star(s))}} \Bigg) + f(n_t, \widehat{\phi}_t),
\end{split}
\label{eq:sufficient_stopping_time}
\end{equation}
where $n_t = (n_t(s,a))_{(s,a) \in \scal \times \acal}$ denotes the number of visits vector. Note that when the terms $(\widehat{T}_i)_{1\leq i\leq 4}$ are bounded and $\underset{t \to \infty}{\lim}\ n_t(s,a) = \infty$ , which we will soon establish, then we have $\underset{t \to \infty}{\lim}\ f(n_t, \widehat{\phi}_t) = 0$.  \\

Next define the convergence event:
\begin{equation*}
\mathcal{C} = \Big\{\forall (s,a) \in \scal\times\acal, \lim_{t \to \infty}\frac{n_t(s,a)}{t} = \overline{\omega}_{s,a},\ \widehat{\phi}_t \to \phi \Big\}.
\end{equation*}
Then by assumptions of the theorem and since $\forall (s,a) ,\ \overline{\omega}_{s,a} > 0$, we have $ \underset{t \to \infty}{\lim}\ n_t(s,a) = \infty$ which implies $\PP_{\phi}(\mathcal{C}) = 1$. Under $\mathcal{C}$, by continuity of the involved functionals of the MDP, we have: 
\begin{equation*}
\forall \varepsilon > 0,\ \exists t_1(\varepsilon) \in \mathbb{N},\ \forall t \geq t_1 : 
\begin{cases}
\widehat{\pi}_t^\star = \pi^{*}, \textrm{ as soon as $\norm{\starQ{\widehat{\phi}_t} - \starQ{\phi}}_{\infty} < \Delta_{\min}/2$}, \\
\widehat{T}_{1,t}(s,a) < (1+\varepsilon) T_{1}(s,a), \quad  \forall s,a\neq \pi^\star (s),  \\
\widehat{T}_{2,t}(s,a) < (1+\varepsilon) T_{2}(s,a), \quad  \forall s,a\neq \pi^\star (s),  \\
\widehat{T}_{3,t} \leq (1+\varepsilon) T_{3}, \\
\widehat{T}_{4,t} \leq (1+\varepsilon) T_{4}, \\

n_t(s,a)/t \geq (1-\varepsilon) \overline{\omega}_{s,a},\ \forall s,a\neq \pi^\star (s), \\
n_t(s,\widehat{\pi}_t^\star(s))/t \geq (1-\varepsilon) \overline{\omega}_{s,\pi^{*}(s)},\ \forall s \in \scal,\\
f(n_t,\widehat{\phi}_t) \leq \varepsilon.
\end{cases}
\end{equation*}
Thus when $t \geq t_1(\varepsilon)$, inequality (\ref{eq:sufficient_stopping_time}) implies:
\begin{equation}
\begin{split}
\textrm{LHS}_t \leq \sqrt{\frac{(1+\varepsilon)\log(\delta')}{(1-\varepsilon) t}} \Bigg(\underset{s,a\neq \pi^\star (s)}{\max} \frac{\sqrt{T_1(s,a;\phi)}+ \sqrt{T_2(s,a;\phi)}}{\sqrt{\overline{\omega}_{s,a}}} +\ \underset{s\in \scal}{\max} \frac{\sqrt{T_3(\phi)}+ \sqrt{T_4(\phi) }}{\sqrt{\overline{\omega}_{s,\pi^\star(s)}}} \Bigg) + \varepsilon.     
\end{split}
\label{eq:sufficient_stopping_time_2}
\end{equation}
Next we define :
\begin{equation}
\begin{split}
    t_2(\delta, \varepsilon) &= \inf\Bigg\{ t > 0\ \bigg|\ 
    \sqrt{\frac{(1+\varepsilon)\log(\delta')}{(1-\varepsilon) t}} \Bigg(\underset{s,a\neq \pi^\star (s)}{\max} \frac{\sqrt{T_1(s,a;\phi)}+ \sqrt{T_2(s,a;\phi)}}{\sqrt{\overline{\omega}_{s,a}}}\\
    &\ \quad \quad \quad+ \underset{s\in \scal}{\max} \frac{\sqrt{T_3(\phi)}+ \sqrt{T_4(\phi) }}{\sqrt{\overline{\omega}_{s,\pi^\star(s)}}} \Bigg) \leq 1 - \varepsilon \Bigg\}\\
    &= \frac{(1+\varepsilon) \log(\delta')}{(1-\varepsilon)^3 } \Bigg(\underset{s,a\neq \pi^\star (s)}{\max} \frac{\sqrt{T_1(s,a;\phi)}+ \sqrt{T_2(s,a;\phi)}}{\sqrt{\overline{\omega}_{s,a}}} +\ \underset{s\in \scal}{\max} \frac{\sqrt{T_3(\phi)}+ \sqrt{T_4(\phi) }}{\sqrt{\overline{\omega}_{s,\pi^\star(s)}}} \Bigg)^2.
\end{split}
\label{eq:def_t_2}
\end{equation}
Combining (\ref{eq:sufficient_stopping_time_2}) and (\ref{eq:def_t_2}), we have for $ t \geq \max(t_1(\varepsilon), t_2(\delta,\varepsilon))$, $LHS_t \leq 1$. Therefore:
\begin{equation*}
\begin{split}
\tau_{\delta} &\leq \max\left(t_1(\varepsilon), t_2(\varepsilon,\delta) \right) \\
& = \max\left(t_1(\varepsilon),\frac{(1+\varepsilon) \log(\delta')}{(1-\varepsilon)^3 } \Bigg(\underset{s,a\neq \pi^\star (s)}{\max} \frac{\sqrt{T_1(s,a;\phi)}+ \sqrt{T_2(s,a;\phi)}}{\sqrt{\overline{\omega}_{s,a}}} +\ \underset{s\in \scal}{\max} \frac{\sqrt{T_3(\phi)}+ \sqrt{T_4(\phi) }}{\sqrt{\overline{\omega}_{s,\pi^\star(s)}}} \Bigg)^2 \right).
\end{split}
\end{equation*}
Thus $\forall \delta \in (0,1),\ \tau_{\delta}$ is finite on $\mathcal{C}$ and we have: $$\forall \varepsilon > 0,\ \underset{\delta \to 0}{\limsup} \frac{\tau_{\delta}}{\log(1/\delta)} \leq \frac{1+\varepsilon}{(1-\varepsilon)^3 } \Bigg(\underset{s,a\neq \pi^\star (s)}{\max} \frac{\sqrt{T_1(s,a;\phi)}+ \sqrt{T_2(s,a;\phi)}}{\sqrt{\overline{\omega}_{s,a}}} +\ \underset{s\in \scal}{\max} \frac{\sqrt{T_3(\phi)}+ \sqrt{T_4(\phi) }}{\sqrt{\overline{\omega}_{s,\pi^\star(s)}}} \Bigg)^2. $$
Taking the limit when $\varepsilon \to 0$, we get:
$$ \underset{\delta \to 0}{\limsup} \frac{\tau_{\delta}}{\log(1/\delta)} \leq \Bigg(\underset{s,a\neq \pi^\star (s)}{\max} \frac{\sqrt{T_1(s,a;\phi)}+ \sqrt{T_2(s,a;\phi)}}{\sqrt{\overline{\omega}_{s,a}}} +\ \underset{s\in \scal}{\max} \frac{\sqrt{T_3(\phi)}+ \sqrt{T_4(\phi) }}{\sqrt{\overline{\omega}_{s,\pi^\star(s)}}} \Bigg)^2. $$
We conclude by applying Lemma \ref{lemma:technical_bound}.
\end{proof}

\subsection{Proof of Theorem \ref{Theorem:Upper bound in expectation}}
For a kernel $u$ in $\mathbb{R}^{S\times SA}$, we define the norm $\norm{u}_{1,\infty} \triangleq \underset{(s,a)\in \scal\times\acal}{\max}\ \underset{s' \in \scal}{\sum} | u(s'|s,a) |$. Next, we define the following distance on MDPs:
$$\norm{\psi - \phi} = \underset{s,a}{\max}\left(\norm{q_{\psi}(.|s,a)- q_{\phi}(.|s,a)}_{1}\ \vee\ \norm{p_{\psi}(.|s,a)- p_{\phi}(.|s,a)}_{1}\right).$$
Based on this distance, we can define balls on the set of MDPs: 
$$
\mathcal{B}_{\norm{.}}(\phi,\xi) \triangleq \{ \psi\ : \ \norm{\psi - \phi} \leq \xi \}.
$$
Let $\varepsilon > 0$. By recursively bounding Bellman operator, one can prove that $\starQ{}$ is Liptschitz w.r.t. rewards and transitions:
\begin{align*}
 \norm{Q_{\phi}^\star-Q_{\psi}^\star}_{\infty} &\leq \left(1+\frac{1}{1-\gamma} \right) \left( \norm{r_{\phi}-r_{\psi}}_{\infty} +  \frac{\gamma}{(1-\gamma)}\norm{p_{\phi}-p_{\psi}}_{1,\infty} \right)\\
 &\leq \left(1+\frac{1}{1-\gamma} \right) \left( \underset{s,a}{\max}\norm{q_{\psi}(.|s,a)- q_{\phi}(.|s,a)}_{1} +  \frac{\gamma}{(1-\gamma)}\norm{p_{\phi}-p_{\psi}}_{1,\infty} \right).
\end{align*}
Thus, there exists $\xi = \xi(\varepsilon) > 0$ such that: 
\begin{equation*}
 \forall \psi \in \mathcal{B}_{\norm{.}}(\phi,\xi), \quad 
\norm{Q_{\phi}^\star-Q_{\psi}^\star}_{\infty} < \Delta_{\min}/2\ \textrm{ and }\ \underset{s,a}{\max} |\overline{\omega}_{s,a}(\psi) - \overline{\omega}_{s,a}(\phi)| \leq \varepsilon.
\end{equation*}
Crucially, the first inequality implies that $\pi^\star_{\psi} = \pi^\star_{\phi}$. For $T \in \mathbb{N}$, consider the concentration event:
$$ \mathcal{E}_T = \bigcap_{t = T^{1/4}}^{T} \left( \widehat{\phi}_{t} \in \mathcal{B}_{\norm{.}}(\phi,\xi) \right).$$
We will be using the following technical lemmas. The first corresponds to Lemma 20 in \cite{garivier16a}, which we reformulate in our case by replacing the number of arms of the bandit by the number of (state, action) pairs of the MDP.
\begin{lemma}
There exists a constant $T_{\varepsilon}$ such that for $T \geq T_{\varepsilon}$, it holds on $\mathcal{E}_T$, for C-Tracking:
$$\forall t \geq T_{\varepsilon},\ \underset{s,a}{\max}\bigg|\frac{n_t(s,a)}{t} - \overline{\omega}_{s,a}\bigg| \leq 3(SA -1) \varepsilon. $$
\label{lemma:concentration of C Tracking proportions}
\end{lemma}
The second lemma is a concentration inequality similar to that of Lemma 19 in \cite{garivier16a} (we defer its proof to the end of this appendix).
\begin{lemma}
Denote by $\mathcal{E}_T^c$ the complementary of the event $\mathcal{E}_T$. There exists two constants $B,C$ (that depend on $\phi$ and $\varepsilon$) such that: 
$$\forall T \geq 1, \PP\left(\mathcal{E}_T^c \right) \leq BT\exp(-CT^{1/8}). $$
\label{lemma:concentration of phi}
\end{lemma}
Recall inequality (\ref{eq:sufficient_stopping_time}), which gives an upper bound of the left-hand-side of the stopping condition: 

\begin{equation*}
\textrm{LHS}_t \leq \sqrt{\log(\delta')} \Bigg( \underset{s,a\neq \widehat{\pi}_{t}^{\star}(s)}{\max} \frac{\sqrt{\widehat{T_1}(s,a)}+ \sqrt{\widehat{T_2}(s,a)}}{\sqrt{n_t(s,a)}}\ +\ \underset{s\in \scal}{\max} \frac{\sqrt{\widehat{T_3}}+ \sqrt{\widehat{T_4} }}{\sqrt{n_t(s,\widehat{\pi}_t^\star(s))}} \Bigg) + f(n_t, \widehat{\phi}_t)
\end{equation*}
where $f(.,.)$ is a continuous function in both arguments.
Define:
\begin{equation*}
\begin{cases}

D(\phi,\varepsilon) = \underset{
\footnotesize{ \begin{split}
&\quad \quad \quad \quad \psi \in \mathcal{B}_{\norm{.}}(\phi,\xi(\varepsilon))\\
&\quad \quad \norm{\omega' - \omega(\phi)} \leq 3(SA - 1) \varepsilon
\end{split}} }{\sup} \displaystyle{\underset{s,a\neq \pi^\star (s)}{\max}\ \frac{\sqrt{T_1(s,a;\psi)}+\sqrt{T_2(s,a;\psi)}}{\sqrt{\omega_{s a}'}}}, \\ \\

E(\phi,\varepsilon) = \underset{
\footnotesize{ \begin{split}
&\quad \quad \quad \quad \psi \in \mathcal{B}_{\norm{.}}(\phi,\xi(\varepsilon))\\
&\quad \quad \norm{\omega' - \omega(\phi)} \leq 3(SA - 1) \varepsilon
\end{split}} }{\sup} \displaystyle{\underset{s\in \scal}{\max}\ \frac{\sqrt{T_3(\psi)}+\sqrt{T_4(\psi)}}{\sqrt{\omega_{s,\pi^\star(s)}'}}}, \\ \\

F(\phi,\varepsilon,t) = \underset{
\footnotesize{ \begin{split}
&\quad \quad \quad \quad \psi \in \mathcal{B}_{\norm{.}}(\phi,\xi(\varepsilon))\\
&\quad \quad \norm{\omega' - \omega(\phi)} \leq 3(SA - 1) \varepsilon
\end{split}} }{\sup} f(t\times\omega',\psi). \\ \\
\end{cases}
\end{equation*}

For $T\geq T_{\varepsilon}$, on the event $\mathcal{E}_T$, we have: $\forall t \geq T^{1/4}, \quad \widehat{\pi}_t^\star = \pi^\star$, and using Lemma \ref{lemma:concentration of C Tracking proportions}, $\norm{\frac{n_t(s,a)}{t} - \overline{\omega}_{s,a}}_{\infty} \leq 3(SA - 1)\varepsilon$.
Therefore, for the stopping condition $\textrm{LHS}_t \leq 1$ to be satisfied, it is sufficient to have:
\begin{equation}
\frac{\sqrt{\log(\delta')}}{\sqrt{t}}\bigg(D(\phi,\varepsilon) + E(\phi,\varepsilon)\bigg) + F(\phi,\varepsilon,t) \leq 1.
\label{eq:sufficient-stopping-time-3}
\end{equation}
By Lemma \ref{lemma:F_goes_to_zero}, $\underset{t \to \infty}{\lim}F(\phi,\varepsilon,t) = 0 $. Hence, we can define the following times :
\begin{equation*}
\begin{cases}
\begin{split}
t_1(\phi,\varepsilon,\eta,\delta) &= \inf\Bigg\{t>0\ |\ \forall x>t, \quad \frac{\sqrt{\log(\delta')}}{\sqrt{x}}\bigg(D(\phi,\varepsilon) + E(\phi,\varepsilon)\bigg) \leq 1-\eta \Bigg\} \\
& = \frac{\log(\delta')\bigg(D(\phi,\varepsilon) + E(\phi,\varepsilon)\bigg)^2}{(1-\eta)^2}
\end{split},\\ \\
t_2(\phi,\varepsilon,\eta) = \inf\Bigg\{t>0\ |\ \forall x>t, \quad F(\phi,\varepsilon,t) \leq \eta \Bigg\}. \\
\end{cases}
\end{equation*}
It is easy to see that for $T \geq \max(T_{\varepsilon },t_1,t_2)$, condition (\ref{eq:sufficient-stopping-time-3}) is verified and consequently: $\tau_{\delta} \leq T$. In other words, we just proved that:
$$ \forall T \geq \max(T_{\varepsilon },t_1,t_2), \quad \mathcal{E}_T \subset (\tau_{\delta} \leq T).$$
Therefore: 
\begin{equation*}
\begin{split}
\EE_{\phi}[\tau_{\delta}] &= \sum_{T=1}^{\infty} \PP(\tau_{\delta} > T) \\ \\
&\leq \sum_{T=1}^{\max(T_{\varepsilon },t_1,t_2)} 1\ + \sum_{T = \max(T_{\varepsilon },t_1,t_2)}^{\infty} \PP(\mathcal{E}_T^c) \\ \\
&\leq T_{\varepsilon } + t_1(\phi,\varepsilon,\eta,\delta) + t_2(\phi,\varepsilon,\eta) + \sum_{T=1}^{\infty} B T\exp(-C T^{1/8}),
\end{split}
\end{equation*}
where the last inequality comes from Lemma \ref{lemma:concentration of phi}. Thus, $\EE[\tau_{\delta}]$ is finite and we have:
\begin{equation*}
    \limsup_{\delta\to 0} \frac{\EE[\tau_{\delta}]}{\log(1/\delta)} \leq \limsup_{\delta\to 0}\frac{t_1(\phi,\varepsilon,\eta,\delta)}{\log(1/\delta)} = \frac{\bigg(D(\phi,\varepsilon) + E(\phi,\varepsilon)\bigg)^2}{(1-\eta)^2}.
\end{equation*}
Letting $\eta$ and $\varepsilon$ go to zero, and noting that: 
\begin{equation*}
\begin{cases}
\underset{\varepsilon \to 0}{\lim} D(\phi,\varepsilon) = \displaystyle{\underset{s,a\neq \pi^\star (s)}{\max}\ \frac{\sqrt{T_1(s,a;\phi)}+\sqrt{T_2(s,a;\phi)}}{\sqrt{\overline{\omega}_{s,a}}}},  \\ \\
\underset{\varepsilon \to 0}{\lim} E(\phi,\varepsilon) = \displaystyle{ \underset{s\in \scal}{\max}\ \frac{\sqrt{T_3(\phi)}+\sqrt{T_4(\phi)}}{\sqrt{\overline{\omega}_{s,\pi^\star(s)}}} },\\ \\
\displaystyle{\bigg(\underset{s,a\neq \pi^\star (s)}{\max}\ \frac{\sqrt{T_1(s,a;\phi)}+\sqrt{T_2(s,a;\phi)}}{\sqrt{\overline{\omega}_{s,a}}} + \underset{s\in \scal}{\max}\ \frac{\sqrt{T_3(\phi)}+\sqrt{T_4(\phi)}}{\sqrt{\overline{\omega}_{s,\pi^\star(s)}}} \bigg)^2 } \leq 4U(\phi),\ \textrm{ (Lemma \ref{lemma:technical_bound}),}

\end{cases}
\end{equation*}
we get the desired result.

\subsection{Second technical lemma}

\begin{lemma}
Let $\pi^\star = \pi_{\phi}^\star$ and let $y(n,m) = (m-1) + (m-1)\log(1+n/(m-1))$. Define:
\begin{equation*}
\begin{split}
f(n, \psi) &= \underset{s, a\neq \pi^\star(s)}{\max} \frac{\sqrt{T_1(s,a;\psi)y(n(s,a),2)}+ \sqrt{T_2(s,a;\psi) y(n(s,a),S) }}{\sqrt{n(s,a)}} \\
&+ \ \underset{s\in \scal}{\max} \frac{\sqrt{T_3(\psi)y(n(s,\pi^\star(s)),2)}+ \sqrt{T_4(\psi) y(n(s,\pi^\star(s)),S) }}{\sqrt{n(s,\pi^\star(s))}} 
\end{split}
\end{equation*}
and 
\begin{equation*}
F(\phi,\varepsilon,t) = \underset{
\footnotesize{ \begin{split}
&\quad \quad \quad \quad \psi \in \mathcal{B}_{\norm{.}}(\phi,\xi(\varepsilon))\\
&\quad \quad \norm{\omega' - \omega(\phi)} \leq 3(SA - 1) \varepsilon
\end{split}} }{\sup} f(t\times\omega',\psi).   
\end{equation*}
Then, there exists $\varepsilon_0$ such that: $\forall \varepsilon \leq \varepsilon_0,\ \underset{t \to \infty}{\lim} F(\phi,\varepsilon,t) = 0$.
\label{lemma:F_goes_to_zero}
\end{lemma}
\begin{proof}
Define: 
\begin{equation*}
\begin{cases}
T_1(s,a,\phi,\varepsilon) \triangleq \underset{\psi \in \mathcal{B}_{\norm{.}}(\phi,\xi(\varepsilon))}{\sup}\ T_1(s,a;\psi), \\
T_2(s,a,\phi,\varepsilon) \triangleq \underset{\psi \in \mathcal{B}_{\norm{.}}(\phi,\xi(\varepsilon))}{\sup}\ T_2(s,a;\psi), \\
T_3(\phi,\varepsilon) \triangleq \underset{\psi \in \mathcal{B}_{\norm{.}}(\phi,\xi(\varepsilon))}{\sup}\ T_3(\psi), \\
T_4(\phi,\varepsilon) \triangleq \underset{\psi \in \mathcal{B}_{\norm{.}}(\phi,\xi(\varepsilon))}{\sup}\ T_4(\psi). \\
\end{cases}   
\end{equation*}
By continuity of the functionals $(T_i)_{1\leq i \leq 4}$ in $\phi$, there exists $\varepsilon_0 > 0$, such that for all $\varepsilon \leq \varepsilon_0$, the supremums defined above are upper bounded by $M = 2\times\underset{s, a\neq \pi^\star(s)}{\max}(T_1(s,a;\phi),T_2(s,a;\phi),T_3(\phi),T_4(\phi))$. Furthermore, if  $\norm{\omega' - \omega(\phi)} \leq 3(SA - 1)\varepsilon$, then for all $(s,a)$: $ \omega_{sa}(\phi)-3(SA - 1)\varepsilon \leq \omega'_{sa} \leq \omega_{sa}(\phi)+3(SA - 1)\varepsilon$. 
Summing up these inequalities we get, for $\varepsilon$ small enough:
\begin{equation}
\begin{split}
F(\phi,\varepsilon,t) &\leq \sqrt{M} \underset{s, a\neq \pi^\star(s)}{\max} \frac{\sqrt{y(t[\omega_{sa}(\phi)+3(SA - 1)\varepsilon],2)}+ \sqrt{y(t[\omega_{sa}(\phi)+3(SA - 1)\varepsilon],S) }}{\sqrt{t[\omega_{sa}(\phi)-3(SA - 1)\varepsilon]}} \\
&+ \ \underset{s\in \scal}{\max} \frac{\sqrt{y(t[\omega_{s,\pi^\star(s) }(\phi)+3(SA - 1)\varepsilon],2)}+ \sqrt{y(t[\omega_{s,\pi^\star(s) }(\phi)+3(SA - 1)\varepsilon],S) }}{\sqrt{t[\omega_{s,\pi^\star(s) }(\phi)-3(SA - 1)\varepsilon]}}. 
\end{split}
\label{eq:hidieous_expression}
\end{equation}
Since $ \forall a >0\ \forall m\geq2,\ \underset{x\to\infty}{\lim} \frac{\sqrt{y(a x,m)}}{\sqrt{x}} = \underset{x\to\infty}{\lim} \frac{\sqrt{(m-1)+(m-1)\log(1 + ax/(m-1))}}{\sqrt{x}} = 0$, and the maximums in (\ref{eq:hidieous_expression}) are taken over finite sets, then $\underset{t \to \infty}{\lim} F(\phi,\varepsilon,t) = 0.$
\end{proof}

\subsection{Proof of Lemma \ref{lemma:concentration of phi} }
\begin{proof} We have:
\begin{equation*}
\begin{split}
\PP\left(\mathcal{E}_T^c \right) &\leq \sum_{t = T^{1/4}}^{T} \PP\left(\widehat{\phi}_t \notin \mathcal{B}_{\norm{.}}(\phi,\xi) \right) \\
&\leq \sum_{t = T^{1/4}}^{T} \underset{s,a}{\sum}\ \Bigg[\PP\bigg(\widehat{r}_t(s,a) - r(s,a) > \xi \bigg) + \PP\bigg(\widehat{r}_t(s,a) - r(s,a) < -\xi \bigg)\\
&\quad \quad \quad  \quad \quad \quad + \underset{s'}{\sum}\ \PP\bigg(\widehat{p}_t(s'|s,a) - p(s'|s,a) > \xi/S \bigg) + \PP\bigg(\widehat{p}_t(s'|s,a) - p(s'|s,a) < -\xi/S \bigg) \Bigg].
\end{split}
\end{equation*}
Let $T$ be such that $T^{1/4} \geq (SA)^2$. Then for $t \geq T^{1/4}$, we have $\forall (s,a), \quad n_t(s,a) \geq (\sqrt{t} -S A/2)_{+} - 1 \geq \sqrt{t} - S A$. Therefore, using a union bound and Chernoff inequality, one can write:
\begin{equation*}
\begin{split}
\PP\bigg(\widehat{p}_t(s'|s,a) - p(s'|s,a) > \xi/S \bigg) &= \PP\bigg(\widehat{p}_t(s'|s,a) - p(s'|s,a) > \xi/S,\ n_t(s,a) \geq \sqrt{t} - S \bigg) \\
&\leq \sum_{t'=\sqrt{t} - SA}^{t} \PP\bigg(\widehat{p}_t(s'|s,a) - p(s'|s,a) > \xi/S,\ n_t(s,a) = t' \bigg) \\
&\leq  \sum_{t'=\sqrt{t} - SA}^{t} \exp\bigg(-t'\cdot \kl\big(p(s'|s,a)+\xi/S,\ p(s'|s,a)\big)\bigg)\\
&\leq \frac{ \exp\bigg(-(\sqrt{t} - SA)\kl\big(p(s'|s,a)+\xi/S,\ p(s'|s,a)\big)\bigg)}{1 - \exp\bigg(-\kl\big(p(s'|s,a)+\xi/S,\ p(s'|s,a)\big)\bigg)}.
\end{split}
\end{equation*}
Using the same reasoning, we can prove that: 
\begin{equation*}
\begin{cases}
\PP\bigg(\widehat{p}_t(s'|s,a) - p(s'|s,a) < -\xi/S \bigg) \leq \displaystyle{\frac{ \exp\bigg(-(\sqrt{t} - SA)\kl\big(p(s'|s,a)-\xi/S,\ p(s'|s,a)\big)\bigg)}{1 - \exp\bigg(-\kl\big(p(s'|s,a)-\xi/S,\ p(s'|s,a)\big)\bigg)}}, \\
\PP\bigg(\widehat{r}_t(s,a) - r(s,a) > \xi \bigg) \leq \displaystyle{\frac{ \exp\bigg(-(\sqrt{t} - SA)\kl\big(r(s,a)+\xi,\ r(s,a)\big)\bigg)}{1 - \exp\bigg(-\kl\big(r(s,a)+\xi,\ r(s,a)\big)\bigg)}}, \\
\PP\bigg(\widehat{r}_t(s,a) - r(s,a) < -\xi \bigg) \leq \displaystyle{\frac{ \exp\bigg(-(\sqrt{t} - SA)\kl\big(r(s,a)-\xi,\ r(s,a)\big)\bigg)}{1 - \exp\bigg(-\kl\big(r(s,a)-\xi,\ r(s,a)\big)\bigg)}}.
\end{cases}
\end{equation*}
Thus, for the following choice of constants 
\begin{equation*}
\begin{split}
C = \underset{s,a}{\min}\Bigg(& \kl\big(r(s,a)-\xi,\ r(s,a)\big)\ \wedge\ \kl\big(r(s,a)+\xi,\ r(s,a)\big)\\
&\wedge\ \underset{s'}{\min}\bigg(\kl\big(p(s'|s,a)-\xi/S,\ p(s'|s,a)\big)\ \wedge\ \kl\big(p(s'|s,a)+\xi/S,\ p(s'|s,a)\big) \bigg) \Bigg)
\end{split}
\end{equation*} 
and 
\begin{equation*}
\begin{split}
B =& \underset{s,a}{\sum}\ \Bigg( \frac{ \exp\bigg(SA \cdot \kl\big(r(s,a)+\xi,\ r(s,a)\big)\bigg)}{1 - \exp\bigg(-\kl\big(r(s,a)+\xi,\ r(s,a)\big)\bigg)} + \frac{ \exp\bigg(SA \cdot \kl\big(r(s,a)-\xi,\ r(s,a)\big)\bigg)}{1 - \exp\bigg(-\kl\big(r(s,a)-\xi,\ r(s,a)\big)\bigg)} \\
&+ \underset{s'}{\sum}\ \Bigg[\frac{ \exp\bigg(SA \cdot \kl\big(p(s'|s,a)+\xi/S,\ p(s'|s,a)\big)\bigg)}{1 - \exp\bigg(-\kl\big(p(s'|s,a)+\xi/S,\ p(s'|s,a)\big)\bigg)} + \frac{ \exp\bigg(SA \cdot \kl\big(p(s'|s,a)-\xi/S,\ p(s'|s,a)\big)\bigg)}{1 - \exp\bigg(-\kl\big(p(s'|s,a)-\xi/S,\ p(s'|s,a)\big)\bigg)} \Bigg] \Bigg),
\end{split}
\end{equation*}
we have 
\begin{equation*}
\begin{split}
\PP\left(\mathcal{E}_T^c \right) \leq \sum_{t = T^{1/4}}^{T} B \exp(-C \sqrt{t}) \leq B T \exp(-C T^{1/8}).
\end{split}
\end{equation*}
\end{proof}

\section{Comparison of KLB-TS and BESPOKE:}\label{sec:compare}

\subsection{Design principles}

As KLB-TS, BESPOKE is an algorithm that adapts its sampling strategy to the learnt MDP. The two algorithms have however different objectives: BESPOKE aims at returning an $\varepsilon$-optimal policy. BESPOKE starts with an intialization phase where each (state, action) pair is sampled $n_{\min} = \frac{2\times 625^2 \times \gamma^2 \times S \times \log(1/\delta)}{(1-\gamma)^2}$ times. After this first phase, the algorithm enters an inner loop. Each iteration of the loop aims at halving the sub-optimality gap 
$\norm{\starV{\phi} - V_{\phi}^{\widehat{\pi}^*}}_{\infty}$ of the empirical best policy. The algorithm iterates until the gap becomes smaller than $\varepsilon$. At the beginning of each iteration, the algorithm solves a convex program whose solution provides the numbers of times each (state, action) pair should be sampled in this iteration. The program minimizes a weighted sum of "confidence intervals" of rewards and transitions estimates at each (state, action) pair, subject to a maximum budget constraint. This objective is known, thanks to the Simulation Lemma\footnote{see Lemma 2 in \cite{BESPOKE2019}}, to be an upper bound of the sub-optimality gap of the empirical optimal policy. BESPOKE uses a doubling trick to compute the maximum budget for each iteration (this budget is defined so that the gap is halved).
We note the following important differences between KLB-TS and BESPOKE.
\begin{enumerate}
\item KLB-TS does not need to solve any convex program to update its sampling strategy, because given an estimate of the MDP, this strategy is explicit. 
\item It is also worth noting that the initialization phase of BESPOKE is extremely long: $\frac{2\times 625^2 \times \gamma^2 \times S^2A \times \log(1/\delta)}{(1-\gamma)^2}$ samples must be gathered. During this phase, the algorithm is not adaptive at all. As we have shown in our numerical experiments, even with small state and action spaces, the initialization phase constitutes a very large proportion of the sample complexity -- which makes the algorithm less adaptive than it seems, and really leads to poor performance. KLB-TS has a much smaller initialization phase and is really adaptive. On Figure \ref{fig:experiment11}, we see that BESPOKE's large sample complexity is mainly due to the constant term corresponding to the minimum number of samples it allocates to each (state, action) pair in the initialization phase. Note that this minimum number of samples cannot be avoided as it is necessary to ensure that BESPOKE halves the accuracy of the empirical policy after each iteration\footnote{see Lemma 16 and the proof of Theorem 1 in \cite{BESPOKE2019}.}
\item BESPOKE's stopping rule is suited to identify $\varepsilon-$optimal policies. Unless it has access an oracle revealing $\Delta_{\min}$, it cannot perform best policy identification.
\end{enumerate}

\begin{figure}[h]
    \centering
        \includegraphics[width=0.5\linewidth]{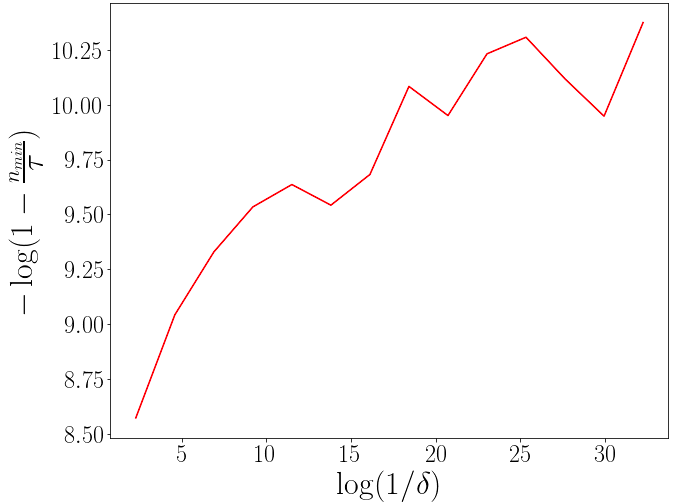}
             \caption{Comparing BESPOKE initialization phase duration $n_{\min}$ to its total sample complexity $\tau$: $-\log(1-\frac{n_{\min}}{\tau})$ as a function of $\log(1/\delta)$. }
         \label{fig:experiment11}
\end{figure}

\subsection{Theoretical guarantees of BESPOKE and KLB-TS}

Theorem 2 in \cite{BESPOKE2019} states that with a probability at least $1-\delta$, the sample complexity of best-policy identification using BESPOKE with $\varepsilon=\delta_{\min}$ is upper bounded by\footnote{$\Tilde{\ocal}(.)$ is used to indicate a quantity that depends on (.) up to a \textit{polylog} expression at most polynomial in $ S, A, \frac{1}{1-\gamma}, \frac{1}{\delta}$.}: 
\begin{equation*}
\begin{split}
\tau_{\delta} = \Tilde{\ocal}\Bigg(& \sum_{s,a\neq \pi^\star (s)} \bigg(\frac{\textrm{Var}[R(s,a)] + \gamma^2 \textrm{Var}_{p(s,a)}[\starV{\phi}]}{\Delta_{sa}^2}+ \frac{1}{(1-\gamma) \Delta_{sa}}\bigg)\\
&+ \sum_{s \in \scal} \min\bigg\{\frac{1}{(1-\gamma)^3 \Delta_{\min}^2},\quad \frac{\textrm{Var}[R(s,\pi^*(s))] + \gamma^2 \textrm{Var}_{p(s,\pi^*(s))}[\starV{\phi}]}{\Delta_{\min}^2} + \frac{1}{(1-\gamma)^2 \Delta_{\min}} \bigg\}+ \frac{S^2 A}{(1-\gamma)^2}\Bigg).
\end{split}
\end{equation*}

In contrast, the sample complexity of KLB-TS scales as:
\begin{equation*}
\begin{split}
\tau_{\delta}  = \ocal\Bigg(&\sum_{s, a\neq \pi^\star(s)}\bigg(\max\bigg\{\frac{ \textrm{Var}_{p(s,a)}[\starV{\phi}]}{\Delta_{sa}^2},\frac{ \textrm{MD}_{p(s,a)}[\starV{\phi}]^{4/3}}{\Delta_{sa}^{4/3}}\bigg\} + \frac{1}{\Delta_{sa}^2} \bigg)\\
&+ S \times \min\bigg\{\frac{1}{(1-\gamma)^3 \Delta_{min}^2}, \max\bigg\{\frac{ \textrm{Var}_{max}^*[\starV{\phi}]}{(1-\gamma)^2 \Delta_{\min}^2}, \frac{ \textrm{MD}_{max}^*[\starV{\phi}]^{4/3}}{(1-\gamma)^{4/3} \Delta_{\min}^{4/3}}\bigg\}\bigg\} + \frac{S}{(1-\gamma)^2 \Delta_{\min}^2} \Bigg) \log(1/\delta)\\
&+ o(\log(1/\delta)).
\end{split}
\end{equation*}
From the above upper bounds, we can make the following comments:
\begin{enumerate}
\item Both bounds depend on functionals of the particular MDP to be learnt, such as the minimum gap, the variance or maximum deviations of value functions. This means that BESPOKE and KLB-TS can adapt to the hardness of the problem, and in particular perform significantly better than minimax approaches when the MDP is easy (e.g. when the minimum gap is high or when the variances of the value function is low).
\item In the worst case, both sample complexities scale at most as $\Tilde{\ocal}\bigg(\frac{S A}{\Delta_{\min}^2 (1-\gamma)^3}\bigg)$, which corresponds to the minimax bound.  
\item When the rewards have strictly positive variances, then the two upper bounds are very similar, except for the large constant term $\frac{S^2 A \log(1/\delta)}{(1-\gamma)^2}$ for BESPOKE which comes from its very long initialization phase. We believe that this constant term makes BESPOKE impractical.
\item While BESPOKE's bound has the advantage of being non-asymptotic, it only holds with probability $1-\delta$. In contrast, KLB-TS comes with an asymptotic bound on the expected sample complexity, which we also proved to be finite for all confidence levels $\delta$.
\end{enumerate}

\end{document}